\documentclass[11pt,letterpaper]{article}
\usepackage[hypertexnames=false]{hyperref}
\hypersetup{
     colorlinks   = true,
     linkcolor = blue,
     urlcolor    = blue,
	 citecolor = blue
}
\usepackage{url}
\usepackage{authblk}

\usepackage{graphicx}
\usepackage{natbib}
\usepackage[margin=1in]{geometry}
\usepackage[utf8]{inputenc} 
\usepackage[T1]{fontenc}    
\usepackage{url}            
\usepackage{booktabs}       
\usepackage{amsfonts}       
\usepackage{nicefrac}       
\usepackage{microtype}      
\usepackage{xcolor}         
\usepackage{pgfplots}
\usepackage{subcaption}
\pgfplotsset{compat=1.18}

\definecolor{plotblue}{RGB}{31,119,180}
\usepackage{tikz}

\usepackage{amsmath,amsfonts,amssymb}
\usepackage{tikz}
\usepackage{pgfplots}
\usepackage{subcaption}

\pgfplotsset{compat=1.18}
\usepackage{aliascnt}
\usepackage{amsthm}
\usepackage{bm}
\usepackage{enumitem,array,booktabs}
\usepackage{algorithm}
\usepackage{algorithmic}
\pgfplotsset{compat=1.18}
\usepgfplotslibrary{groupplots}
\usetikzlibrary{calc}
\definecolor{avgcolor}{HTML}{BE0AFF}
\definecolor{calcolor}{HTML}{147DF5}
\definecolor{unconstcolor}{HTML}{A1FF0A}
\definecolor{trrcolor}{HTML}{FF8700}
\pgfplotsset{
  paperaxis/.style={/pgfplots/.cd,
    scale only axis,
    axis lines=box,
    axis line style={black,line width=0.8pt},
    tick align=outside,
    tick style={black,line width=0.5pt},
    major tick length=3pt,
    minor tick length=1.7pt,
    tick label style={font=\fontsize{9}{10}\selectfont},
    label style={font=\fontsize{10}{11}\selectfont},
    ymode=log, log basis y=10,
    ymin=0.8,
    grid=none,
    clip=true,
    unbounded coords=jump,
    legend style={draw=none,fill=none,font=\fontsize{8.5}{10}\selectfont,
      /tikz/every even column/.append style={column sep=6pt}},
    legend cell align=left,
    every axis plot/.append style={line width=1.0pt,mark size=2.4pt},
  },
  avgline/.style={color=avgcolor,mark=*,mark options={solid}},
  calline/.style={color=calcolor,mark=square*,mark options={solid}},
  unconstline/.style={color=unconstcolor,mark=triangle*,mark size=3pt,mark options={solid}},
  trrexpline/.style={color=trrcolor,mark=diamond*,mark size=3pt,mark options={solid}},
  trrmaxline/.style={color=trrcolor,densely dotted,line width=1.4pt,
    mark=diamond,mark size=4pt,mark options={solid,fill=none,line width=0.9pt}},
}
\newcommand{\paperlegend}{%
  \begin{tikzpicture}[baseline=-0.5ex]
    \matrix[column sep=9pt,ampersand replacement=\&]{
      \draw[avgcolor,line width=1pt] (0,0)--(.38,0);
      \fill[avgcolor] (.19,0) circle (2.2pt);
      \node[anchor=west,font=\fontsize{8.5}{10}\selectfont] at (.43,0) {MLE-AVG}; \&
      \draw[calcolor,line width=1pt] (0,0)--(.38,0);
      \fill[calcolor] (.19,0)+(-2pt,-2pt) rectangle +(2pt,2pt);
      \node[anchor=west,font=\fontsize{8.5}{10}\selectfont] at (.43,0) {MLE-CAL}; \&
      \draw[unconstcolor,line width=1pt] (0,0)--(.38,0);
      \fill[unconstcolor] (.19,0)+(0,2.6pt)--+(-2.4pt,-1.6pt)--+(2.4pt,-1.6pt)--cycle;
      \node[anchor=west,font=\fontsize{8.5}{10}\selectfont] at (.43,0) {MLE-UNCONST}; \&
      \draw[trrcolor,line width=1pt] (0,0)--(.38,0);
      \fill[trrcolor] (.19,0)+(0,2.6pt)--+(2.6pt,0)--+(0,-2.6pt)--+(-2.6pt,0)--cycle;
      \node[anchor=west,font=\fontsize{8.5}{10}\selectfont] at (.43,0) {TRR: expected}; \&
      \draw[trrcolor,densely dotted,line width=1.4pt] (0,0)--(.38,0);
      \draw[trrcolor,line width=.8pt] (.19,0)+(0,3.2pt)--+(3.2pt,0)--+(0,-3.2pt)--+(-3.2pt,0)--cycle;
      \node[anchor=west,font=\fontsize{8.5}{10}\selectfont] at (.43,0) {TRR: maximum}; \\
    };
  \end{tikzpicture}%
}

\newcounter{procedure}

\usepackage[nameinlink, capitalize]{cleveref}

\usepackage{natbib}
\usepackage{bbm}
\usepackage{tabulary,multirow}
\usepackage{makecell}
\usepackage{dsfont}
\usepackage{tcolorbox}
\usepackage{thmtools}
\usepackage{xpatch}
\xpatchcmd{\proof}{\itshape}{\scshape}{}{}
\usepackage{tablefootnote}

\theoremstyle{plain}
\newtheorem{theorem}{Theorem}[section]
\newtheorem{proposition}[theorem]{Proposition}
\newtheorem{lemma}[theorem]{Lemma}
\newtheorem{corollary}[theorem]{Corollary}

\theoremstyle{definition}
\newtheorem{definition}[theorem]{Definition}

\theoremstyle{remark}

\crefname{theorem}{Theorem}{Theorems}
\crefname{proposition}{Proposition}{Propositions}
\crefname{corollary}{Corollary}{Corollaries}

\DeclareMathOperator*{\argmin}{arg\,min}

\newcommand{\ra}{\rightarrow}

\newcommand{\Unif}{\mathrm{Unif}}

\newcommand{\1}{\mathds{1}}

\newcommand{\EEs}[2]{\mathbb{E}_{#1}\left[#2\right]}

\newcommand{\abs}[1]{\left|#1\right|}

\newcommand{\cA}{\mathcal{A}}

\newcommand{\cC}{\mathcal{C}}
\newcommand{\cD}{\mathcal{D}}

\newcommand{\cI}{\mathcal{I}}

\newcommand{\cM}{\mathcal{M}}

\newcommand{\cP}{\mathcal{P}}

\newcommand{\cT}{\mathcal{T}}

\newcommand{\cU}{\mathcal{U}}

\newcommand{\eps}{\varepsilon}
\renewcommand{\epsilon}{\varepsilon}
\renewcommand{\hat}{\widehat}
\renewcommand{\tilde}{\widetilde}
\renewcommand{\bar}{\overline}

\newcommand{\nothere}[1]{}

\newcommand{\TV}{\mathrm{TV}}

\newcommand\blfootnote[1]{%
  \begingroup
  \renewcommand\thefootnote{}\footnote{#1}%
  \addtocounter{footnote}{-1}%
  \endgroup
}
\title{How Many Repeated Pairwise Comparisons Are Needed for Ranking under Heterogeneity?}

\date{}
\begin{document}
\author[1]{Shashaank Aiyer}
\author[1]{Han Shao}
\affil[1]{University of Maryland}
\maketitle
\blfootnote{Emails: \href{mailto:saiyer1@umd.edu}{saiyer1@umd.edu}, \href{mailto:hanshao@umd.edu}{hanshao@umd.edu}}
\begin{abstract}
We study ranking models by population-average utility from pairwise
comparisons when preferences vary across users and tasks.
Prior work shows that a single comparison per user can be insufficient
to identify the alternative with the highest average utility, even with
arbitrarily many users \citep{golz2025distortion}.
We investigate how many repeated comparisons within each user-task
context are necessary and sufficient for ranking recovery.
Under a heterogeneous Bradley--Terry model with fixed inverse
temperature, we start with a naive MLE-based algorithm that requires $\Omega(1/\Delta^2)$ repeated comparisons per context to ensure ranking recovery. We then present two MLE-based variants and a randomized Russian Roulette-style algorithm that recover the ranking using
$O(\log(1/\Delta))$ repeated comparisons per context, and we prove that this logarithmic dependence is optimal.
Despite this worst-case requirement, our Russian Roulette algorithm uses only
$O(1)$ comparisons per context in expectation.
Synthetic experiments and semi-synthetic experiments based on Arena
data compare the four algorithms in settings with varying levels of preference heterogeneity and under varying context distributions.
\end{abstract}
\section{Introduction}
In recent years, human feedback in the form of pairwise comparisons has become a key instrument for evaluating and aligning large language models. Rather than asking for quantitative scores, platforms like LM Arena \citep{chatbotarena} present users with two model responses associated with a given prompt and ask which is better. LM Arena uses such comparisons to construct global model leaderboards, while RLHF uses them to learn reward models that guide finetuning \citep{ouyang, christiano}. 

Two features of pairwise human feedback introduce complications to these tasks: the heterogeneity of users and the price of obtaining pairwise comparison data. First, preferences across users and tasks can genuinely differ, making it somewhat impractical to assume a shared global ranking. For example, one user could prefer concise responses for some tasks and detailed explanations for others while another user's preferences are completely flipped. The standard procedure of fitting a single Bradley-Terry (BT) model inherently ignores these discrepancies and treats variation in user preferences as noise surrounding an agreed-upon ranking of models. \citet{golz2025distortion} show quantitatively that the BT-based alignment can fail to optimize for population-level average utility under heterogeneous user preferences. \citet{pukdee2026what} also study this problem, characterizing precise conditions under which heterogeneous preference data admits a BT representation. 

The second limiting feature of pairwise comparison data involves the high cost of obtaining feedback, particularly as AI models scale to increasingly specialized applications and tasks. We believe that this evaluation cost problem appears at an even more granular level than simply obtaining feedback from additional users. In particular, even in instances where a platform is able to collect a large number of pairwise comparisons, it may be difficult to obtain \textit{repeated} comparisons from any one user. For example, \cite{chatbotarena} obtained 240,000 votes in their January 2024 dataset, and these votes were dispersed across 90,000 users, fewer than three votes per user on average. Paired with the fact that users' preferences are heterogeneous, data from one user need not be useful for learning another user's internal model scores. In particular, \cite{golz2025distortion} show that in an observation model where each user only provides one comparison, population-average utility is unrecoverable. Armed with this intuition, we aim to model and study the following central question:

\textit{How many repeated comparisons per user-task context are necessary and sufficient to recover the ranking of models by population-average utility?}

We formalize this question through a contextual Bradley-Terry model under which each user $i$ is imbued with a latent bounded utility function $u_i$, where $u_i(t,m)$ denotes user $i$'s utility derived from model $m$ for a task $t$. We seek to estimate the average utility across users and tasks. Importantly, we do not place any assumptions on the relationship between the individual utility functions across different users as well as the utilities across tasks or models for the same user. For a fixed user and task, comparisons follow a BT model with fixed known parameter $\beta > 0$, and we assume that repeated comparisons of two models within a context are conditionally independent BT draws \citep{jin2020rank, negahban}. The number of evaluation samples then has two distinct components: the number of independently sampled user-task contexts and the number of repeated comparisons allowed for any pair of models within a fixed context. 

\begin{figure}[t]
    \centering
    \includegraphics[width=0.8\textwidth]{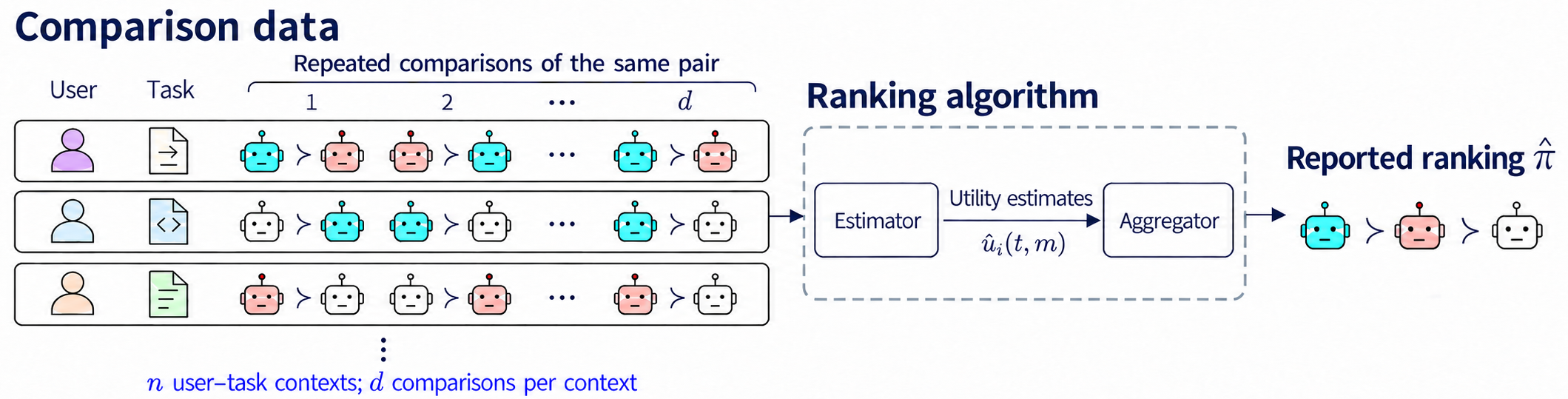}
    \label{fig:introvisual}
    \caption{\textbf{Ranking from repeated pairwise comparisons.}
For each of $n$ independently sampled user-task contexts, we sample one model pair and collect $d$ independent comparisons of that pair. The ranking algorithm consists of an estimator that estimates context-level utilities and an aggregator that combines these estimates to recover the ranking induced by population-average utility.} 
\end{figure}

To see why we explicitly distinguish between sampling fresh contexts and obtaining repeated comparisons within a fixed context, consider a fixed model and user-task pair, and let $L$ denote the latent utility difference between the two models. Then, a single comparison is a Bernoulli observation with mean $\sigma(\beta L)$, where $\sigma(x) = (1+e^{-x})^{-1}$, so sampling more contexts reveals the expected win probability $\mathbb{E}[\sigma(\beta L)]$. However, our target ranking is governed by the different in average utilities, $\mathbb{E}[L]$, whereas many contexts with a single comparison reveal the sign of $\mathbb{E}[\sigma(\beta L)] - 1/2$. Since the BT link function is nonlinear, these two quantities can have different signs. Repeated comparisons within contexts provide information about the individual win probabilities and thereby the latent utility differences. The question then becomes: How many repetitions do we need to recover enough information to estimate the average utility over the population?
\paragraph{Our Contributions.} In studying the previous question, we optimally characterize the number of repeated comparisons for a fixed user-task context needed to rank models under average utility. Let $\Delta$ denote the minimum nonzero gap in the average utility of two models in our class, and let $\delta$ be the desired failure probability. Throughout the paper, $\beta>0$ is fixed and the constants in our results may depend on it. We begin with a natural MLE-based approach and show that naively aggregating its estimates into a ranking can require $\Omega(1/\Delta^2)$ repeated comparisons because of estimation bias. This motivates two improved MLE-based algorithms as well as a randomized algorithm, all of which reduce the repetition budget to $O(\log (1/\Delta))$. We then show via a lower bound that this logarithmic dependence is optimal. In sum, our contributions are four-fold:
\begin{itemize}[leftmargin=*,itemsep=1.5pt,topsep=1.5pt]
\item \textbf{MLE-Based Algorithms.} A natural approach to this problem involves using maximum likelihood estimation to estimate each model's utility for each sampled user-task context. We formally describe and analyze three such algorithms, the first requiring a repetition budget quadratic in $1/\Delta$ and the final two using a repetition budget logarithmic in $1/\Delta$.
\item \textbf{Russian Roulette-Based Algorithm.} We introduce a Russian-Roulette estimator that randomizes the number of repeated comparisons it actually uses within each context. It also uses a repetition budget logarithmic in $1/\Delta$. However, in expectation, it only uses a constant number of repeated comparisons per context. 
\item \textbf{Information-Theoretic Lower Bound on the Repetition Budget.} We provide a lower bound showing that any ranking algorithm must require a repetition budget with logarithmic dependence on $1/\Delta$, proving that our earlier algorithms are optimal.
\item \textbf{Synthetic and Semi-Synthetic Experiments.} We complement our theoretical results with synthetic experiments as well as semi-synthetic, Arena-based experiments comparing our four algorithms. In particular, we investigate how varying the heterogeneity in user utilities and context distribution affects the performance of our algorithms. 
\end{itemize}

\section{Problem Setup} \label{sec:setup}

\paragraph{Data model.} We are given a task set $\cT$, a model set $\cM$, and a user set $\cU$. Suppose users are i.i.d.\ samples from a distribution $\cD$. Each user $i$ is imbued with a distribution over tasks, $\cP_i$, and a utility function $u_i:\cT\times \cM \ra[0,1]$ measuring the performance of each model on each task. The ground-truth ranking is determined by the expected utility. More specifically, for each user $i$ and model $m$, the utility of model $m$ for user $i$ is $u_i(m) = \EEs{t\sim \cP_i}{u_i(t,m)}$. And the expected utility takes expectation over users, i.e., 
    \[
    \bar{u}(m) = \EEs{i\sim \cD}{u_i(m)}.
    \]
Let $\pi^\star$ denote a ground-truth ranking satisfying $\bar{u}(m)>\bar{u}(m') \implies m\succ_{\pi^\star}m'$ for all $m,m'\in\cM$, where $m\succ_{\pi^\star}m'$ means that $m$ is ranked above $m'$ under $\pi^\star$. 

We define the minimum utility gap between any two models in $\cM$ as
\[
    \Delta
    :=
    \min_{\substack{m\neq m'\in\cM}}
    \abs{\bar u(m)-\bar u(m')}.
\]
For simplicity, we assume no ties between models, i.e., $\Delta>0$.

Rather than directly observing cardinal user utilities, we only have access to task-level pairwise comparison data. That is, we sample a user-task context $(i,t)$, and a pair of models $\{m,m'\}\sim \cC$ drawn from a distribution $\cC$, and we ask the user to compare the models' performance under this task. We follow Bradley-Terry to model the comparison outcome. In particular, 
    \[
    \mathbb{P}(m \succ m' | i,t) = \sigma(\beta(u_{i}(t,m)-u_i(t,m'))),
    \]
    where $\sigma(x) = \frac{1}{1+e^{-x}}$ and $\beta > 0$ is the inverse temperature parameter. 

Given positive integers $n,d$, the dataset $S=\{(i^j,t^j,m_1^j,m_2^j,y_1^j,\ldots,y_d^j)\}_{j=1}^n$ is generated as follows. Independently for each $j\in[n]$, we sample a user $i^j\sim\cD$, a task $t^j\sim\cP_{i^j}$, and a model pair $\{m_1^j,m_2^j\}\sim \cC$, where $\cC$ is a model-pair sampling distribution. The two models are labeled according to a fixed ordering of $\cM$. Conditional on this context and model pair, we draw
\[
y_1^j,\ldots,y_d^j
\overset{\mathrm{i.i.d.}}{\sim}
\operatorname{Bernoulli}\!\left(
\sigma\!\left(\beta\bigl(
u_{i^j}(t^j,m_1^j)-u_{i^j}(t^j,m_2^j)
\bigr)\right)
\right),
\]
where $y_s^j=1$ indicates that $m_1^j$ wins the $s$th comparison and $y_s^j=0$ indicates that $m_2^j$ wins. Thus, $n$ is the number of sampled user-task contexts, and $d$ is the number of repeated comparisons of the sampled model pair within each context.

\paragraph{Ranking Recovery and Repetition Budget Complexity.}
A ranking algorithm $\cA$ receives a dataset $S$ and outputs a ranking $\widehat{\pi}_{\cA}=\cA(S)$ over $\cM$. We measure its performance by the ranking error
\[
\mathrm{RankErr}(\widehat{\pi}_{\cA})
:= \1[\widehat{\pi}_{\cA}\ne\pi^\star].
\]

Our goal is to characterize the smallest repetition budget that permits ranking recovery using finitely many sampled contexts.

\begin{definition}[$(\epsilon,\delta)$-repetition budget]\label{def:repbudget}
Fix $\epsilon,\delta\in(0,1)$. For a ranking algorithm $\cA$, its $(\epsilon,\delta)$-repetition budget $d_{\cA}(\epsilon,\delta)$ is the smallest positive integer $d$ for which there exists a finite positive integer $n$ and a model-pair sampling distribution $\cC$ such that, for every problem instance with minimum average utility gap $\Delta\ge\epsilon$, a dataset $S$ generated with parameters $(n,d)$ satisfies
\[
\mathbb{P}\bigl(\mathrm{RankErr}(\cA(S))=0\bigr)\ge 1-\delta.
\]
The probability is over the data generation process and any internal randomness of $\cA$. 
The sample size $n$ may depend on $\cA$, $\epsilon$, $\delta$, and the fixed parameters $\beta$ and $|\cM|$, but not on the particular instance. If no such $d$ exists, set $d_{\cA}(\epsilon,\delta)=+\infty$.
\end{definition}

The optimal $(\epsilon,\delta)$-repetition budget is
\[
d^\star(\epsilon,\delta)
:= \inf_{\cA}d_{\cA}(\epsilon,\delta).
\]
We seek to understand how large $d^\star(\epsilon,\delta)$ is.
\section{Ranking Algorithms and Analyses}
A natural idea for designing algorithms for our problem is to first estimate the utilities within contexts and then use these estimates to produce a ranking. Such an algorithm consists of two components: (1) An \textbf{estimator} that approximates the underlying context-level utility functions $u_i(t,m)$ for all $i,t,m$ and (2) An \textbf{aggregator} that converts these estimates into a ranking. 
In the following, we will propose several different ways of constructing estimators and aggregators and analyze how they affect the overall algorithm's repetition budget. We start with a constrained MLE estimator and a naive aggregation method by simply calculating the empirical utility differences with respect to a reference model, and we show that this simple method requires an $(\eps,\delta)$-repetition budget of $\Omega(1/\eps^2)$.
Developing intuition from the revealed estimation bias of this algorithm, we show that the idea of building algorithms based on maximum likelihood can indeed achieve repetition budget of $O(\log(1/\eps))$ with carefully designed aggregators. 
We also propose a randomized Russian Roulette-based estimator \footnote{The name "Russian Roulette" refers to a randomized Monte Carlo method used to obtain an unbiased estimate of an infinite series using finite computation. The method is attributed to early work by von Neumann and Ulam \citep{kahnmontecarlo,russianroulette}}, obtaining the same repetition budget of $O(\log(1/\eps))$ and further only requiring $O(1)$ repeated comparisons in expectation. Finally, we present a lower bound showing that $\Theta(\log(1/\eps))$ repetition budget is optimal.

\subsection{MLE-Based Algorithms}

A standard approach to ranking from pairwise comparisons is to estimate utilities by maximum likelihood and then rank models based on these estimates. Given a dataset $S$, a candidate utility function $u:\cU\times\cT\times\cM\to[0,1]$ induces the empirical negative log-likelihood
\[
\begin{aligned}
L_S(u)
= -\sum_{j=1}^n\sum_{r=1}^d \Bigl[
&y_r^j\log\sigma\!\left(
\beta\bigl(
u(i^j,t^j,m_1^j)-u(i^j,t^j,m_2^j)
\bigr)
\right) +(1-y_r^j)\log\sigma\!\left(
\beta\bigl(
u(i^j,t^j,m_2^j)-u(i^j,t^j,m_1^j)
\bigr)
\right)
\Bigr].
\end{aligned}
\]
Since the likelihood depends only on within-context utility differences, absolute utility levels are generally not identifiable. We therefore fix an arbitrary reference model $m_0\in\cM$ and parameterize utilities relative to it:
$u(i,t,m)-u(i,t,m_0)\in[-1,1]$.
This normalization leaves the likelihood unchanged and ensures that $u(i,t,m_0)=0$ for every $i\in\cU$ and $t\in\cT$.

\paragraph{Constrained MLE Estimator with Naive Empirical Aggregation.} 

We start by considering the most naive choice of the estimator and aggregator. In particular, now that we are estimating utility \textit{differences} relative to a reference model, our MLE estimate is 
\[
\widehat{v} \in \argmin_{ v:\cU\times\cT\times\cM\to[-1,1], v(i,t,m_0)=0, \forall (i,t)} L_{S}(v).
\]
Importantly, we constrain our estimate to map to $[-1,1]$, respecting the fact that the true utility differences lie in this interval.

Once we have our $\hat{v}$, it remains to aggregate the estimate into a ranking under average utility. The natural choice here is, for any model $m \neq m_0$, to isolate contexts $j$ in which $m$ was compared to $m_0$ and take the empirical average over the estimates $\hat{v}(i^j,t^j,m)$. This avoids having to combine potentially inconsistent estimates obtained from different model pairs, instead providing an estimate for each model's utility relative to the same reference. 
Therefore, we estimate the average utility of $m$ by
\[\frac{\sum_{j\in [n]: \{m_1^j,m_2^j\}=\{m,m_0\}} \hat{v}(i^j,t^j,m)}{|\{j\in [n]: \{m_1^j,m_2^j\}=\{m,m_0\}\}|}\,.\]
The resulting algorithm, $\cA_{\mathrm{MLE-AVG}}$ ranks models in decreasing order of these empirical averages, breaking ties according to the fixed order of $\cM$.

A pitfall surrounding this algorithm is that given a finite repetition budget, the constrained MLE can be a biased estimator of the context-level utility differences. In fact, this bias can be of order $1/\sqrt{d}$. As an example, consider a context where the true utility difference of the compared models within that context is 1. In this case, the constrained MLE estimate is allowed to fall below 1, but any estimate above 1 is truncated to 1. Given only $d$ Bernoulli comparisons, the empirical win frequency of a model has standard deviation of order $1/\sqrt{d}$, resulting in a systematic downward MLE bias of the same order. Then, a standard empirical average over contexts can concentrate around this biased expectation. This leads to a lower bound on the repetition budget that is quadratic in $1/\eps$, as shown in the following result. 
\begin{restatable}{proposition}{mleavg}\label{prop:mleavg}
    Fix $\delta \in (0,1/2)$ and $\eps \in (0,1)$. Then, $\cA_{\mathrm{MLE-AVG}}$ requires 
    \[
     d_{\cA_{\mathrm{MLE-AVG}}}(\eps,\delta) \geq \Omega\left(\frac{1}{\eps^2 }\right). 
    \]
\end{restatable}
Now, the question remains: Is this repetition budget of $\Omega\left(\frac{1}{\eps^2 }\right)$ necessary for any algorithm? Can we improve the repetition budget? In the following, we introduce two different MLE-based algorithms that obtain an improved repetition budget of $O(\log(1/\eps))$.

\paragraph{Unconstrained MLE Estimator with Harmonic Empirical Aggregation.} The undesirable repetition budget of $\cA_{\mathrm{MLE-AVG}}$ is primarily governed by the bias associated with constraining the underlying estimator to lie in the interval $[-1,1]$. Thus, a natural followup is simply to remove this constraint. Call the resulting estimate $\hat{v}_{\text{unc}}$. In this case, for any model $m \neq m_0$ and any context $j$ in which $\{m,m_0\}$ is the sampled pair, the MLE estimate is simply 
\[
\hat{v}_{\text{unc}}(i^j,t^j,m) = \frac{1}{\beta}\log \frac{W_j(m)}{d-W_j(m)},
\]
where $W_j(m)$ denotes the number of times $m$ wins in the $d$ comparisons within the context. We adopt the convention that $\hat{v}_{\text{unc}}(i^j,t^j,m) = \infty$ if $W_j(m) = d$ and $\hat{v}_{\text{unc}}(i^j,t^j,m) = -\infty$ if $W_j(m) = 0$. The key insight here is that , unlike the constrained MLE estimate where some outputs are truncated to 1 or -1, the unconstrained MLE estimate for $m$ is uniquely determined by its number of wins. In particular, $\hat{v}_{\text{unc}}(i^j,t^j,m)$ has a unique closed form if $j$ is a context in which $m$ was compared to $m_0$:
\[
\hat{v}_{\text{unc}}(i^j,t^j,m) = \frac{1}{\beta}\log \frac{W_j(m)}{d-W_j(m)}.
\]
Now, prior to aggregation, we have to account for the fact that the estimates can take infinite values. The naive empirical mean will be infinite or undefined in this case. Thus, we need to find some way to truncate such values. Rather than doing this arbitrarily, we seek a finite transformation of the MLE estimates whose conditional expectation remains close to the underlying utility differences.  In particular, we replace each estimate $(1/\beta)\log(W_j(m)/(d-W_j(m)))$ by 
\[
\frac{1}{\beta}(H_{W_j(m)}-H_{d-W_j(m)}),
\]
where $H_0 = 0, H_k=\sum_{i=1}^k (1/i)$ denotes the $k$th harmonic number. We empirically average after applying this transformation, obtaining our algorithm $\cA_{\mathrm{MLE-UNCONST}}$. The key idea that makes this algorithm work is that the MLE estimates after transformation have expectation 
\[
\frac{1}{\beta}\sum_{k=1}^d \frac{p_j(m)^k-(1-p_j(m))^k}{k},
\]
where $p_j(m) = \sigma(\beta(u_{i^j}(t^j,m) - u_{i^j}(t^j,m_0)))$ denotes the Bradley-Terry win probability for model $m$ in context $j$. Meanwhile, the underlying context-level utility difference also has a series representation:
\[
u_{i^j}(t^j,m) - u_{i^j}(t^j,m_0) = \frac{1}{\beta}\log \frac{p_j(m)}{1-p_j(m)} = \frac{1}{\beta}\sum_{k=1}^{\infty}\frac{p_j(m)^k - (1-p_j(m))^k}{k}
\]
Thus, the expectation of our estimates is just the truncation of this series to $d$ terms, and controlling the bias simply involves controlling the tail, which we show decays exponentially with $d$. 

\paragraph{Constrained MLE Estimator with Calibrated Empirical Aggregation.} Now, the reader may wonder if we can apply a different aggregation method to the constrained MLE estimator and obtain an improved guarantee on the repetition budget.  Is it possible to calibrate these values in order to reduce their bias prior to empirical averaging? Our resulting algorithm, $\cA_{\mathrm{MLE-CAL}}$ exploits exactly this idea. 

The key observation that our calibration mechanism relies on is that for any fixed context $j$, $\widehat{v}(i^j,t^j,m_1^j)$ only takes on $d+1$ values, each associated with the number of times $m_1^j$ wins against $m_2^j$ in the $d$ comparisons within that context. Then, we can calibrate each of these values so that in expectation, the calibrated MLE value is uniformly close to every possible value of $u_{i^j}(t^j,m_1^j) - u_{i^j}(t^j,m_2^j)$. In particular, we show that their worst-case bias decays exponentially with $d$ rather than the $1/\sqrt{d}$ rate of $\cA_{\mathrm{MLE-AVG}}$. Importantly, the calibrated values are allowed to lie outside $[-1,1]$, so they do not suffer from the downward bias of the original MLE. For a formal definition of this algorithm, we refer the reader to \cref{mle-cal-defn}.

In sum, this calibrated MLE-based algorithm starts with biased estimates of the underlying context-level utility differences and corrects these estimates to control the bias. 

The following theorem consolidates our results about the latter two MLE-based algorithms. 
\begin{restatable}{theorem}{mle}\label{prop:mle}
    Fix $\epsilon,\delta \in (0,1)$. Then, the $(\epsilon,\delta)$-repetition budgets for $\cA_{\mathrm{MLE-CAL}}$ and $\cA_{\mathrm{MLE-UNCONST}}$ satisfy \[ d_{\cA_{\mathrm{MLE-CAL}}}(\eps,\delta) \leq O\left(\log\frac{1}{\eps }\right) \] 
    and 
    \[ d_{\cA_{\mathrm{MLE-UNCONST}}}(\eps,\delta) \leq O\left(\log\frac{1}{\eps }\right). \]
\end{restatable}

The following corollary specifies the number of sampled contexts needed for the algorithms to obtain their guarantees. 

\begin{corollary}
There exists a model-pair sampling distribution $\cC$ such that the upper bounds on $d_{\cA_{\mathrm{MLE-CAL}}}(\eps,\delta)$ and $d_{\cA_{\mathrm{MLE-UNCONST}}}(\eps,\delta)$ in \cref{prop:mle} can be realized with 
\[
n = O \left(\frac{|\cM|}{\eps^{2+\kappa_{\beta}}}\log\frac{|\cM|}{\delta} \right) \qquad \text{and} \qquad n = O \left(\frac{|\cM|}{\eps^{2}}\log\frac{|\cM|}{\delta} \right) \qquad
\]
sampled contexts, respectively, where $\kappa_{\beta}>0$ is a fixed constant depending only on $\beta$. Consequently, the total number of pairwise comparisons that $\cA_{\mathrm{MLE-CAL}}$ and $\cA_{\mathrm{MLE-UNCONST}}$ use is 
\[
O\left(\frac{|\cM|}{\eps^{2+\kappa_{\beta}}}\log\frac{1}{\eps}\log\frac{|\cM|}{\delta} \right) \qquad \text{and} \qquad O\left(\frac{|\cM|}{\eps^{2}}\log\frac{1}{\eps}\log\frac{|\cM|}{\delta} \right),
\]
respectively.
\end{corollary}

\subsection{Truncated Russian Roulette (TRR) Estimator}
In this section, we present a randomized algorithm that deviates from the MLE-based approach by instead starting with an unbiased estimator of the context-level utility differences. We show that not only does this algorithm obtain a $\log(1/\eps)$ repetition budget, but its randomized design permits it to only use a constant number of repeated comparisons in expectation. We adopt the same reference model design with reference model $m_0 \in \cM$. 

To start, for a fixed $m \neq m_0$ and context $j$ in which $\{m,m_0\}$ is the sampled pair, we write the series representation of the context-level utility difference between $m$ and $m_0$:
\[
u_{i^j}(t^j,m) - u_{i^j}(t^j,m_0) = \frac{1}{\beta}\sum_{k=1}^{\infty}\frac{p_j(m)^k - (1-p_j(m))^k}{k},
\]
Now, observe that each of the terms in the series can be estimated directly from the comparison outcomes $y_1^j,...,y_d^j$ without having to estimate $p_j(m)$. In particular, $p_j(m)^k$ simply denotes the probability that $m$ beats $m_0$ in the first $k$ comparisons, and $(1-p_j(m))^k$ is analogously the probability that $m_0$ beats $m$ in the first $k$ comparisons. Then, the following statistic is an unbiased estimator of the utility difference between $m$ and $m_0$ for this context: 
\[
\frac{1}{\beta}\sum_{k=1}^{\infty}\frac{\prod_{r=1}^k y_r^j - \prod_{r=1}^k (1-y_r^j)}{k}.
\]

Although this statistic is unbiased, it requires infinitely many comparisons to compute. To overcome this issue, we introduce randomness, using a Russian Roulette estimation technique. In particular, within each context, we independently draw an integer $K_j$ according to a geometric distribution. We then only use the first $K_j$ repeated comparisons and estimate the first $K_j$ terms of the above statistic. Then, scaling by the inverse of the survival probability $\mathbb{P}(K_j \geq k)$ gives the statistic 
\[
\frac{1}{\beta}\sum_{k=1}^{K_j}\frac{\prod_{r=1}^k y_r^j - \prod_{r=1}^k (1-y_r^j)}{k\mathbb{P}(K_j \geq k)},
\]
Importantly, this statistic remains an unbiased estimator of the context-level utility difference but is now computable using finitely many comparisons. However, the sampled $K_j$ is still allowed to be arbitrarily large, which is not compatible with our setting, where we are restricted to $d$ repeated comparisons per context. Hence, we adjust the distribution of $K_j$ to truncate at $d$, but this truncated distribution still has geometrically decaying tails. For a formal description of the estimator and distribution of the $K_j$'s, we refer the reader to \cref{appendix:randomtruncated}.

Although truncating at $d$ causes the statistic to become biased, we show that this bias decays exponentially with $d$. Our resulting algorithm $\cA_{\text{TRR}}$ applies the naive empirical aggregation procedure and requires a repetition budget that is logarithmic in $1/\eps$, as shown in the following result. 

\begin{restatable}{theorem}{randomtruncated}\label{thm:randomtruncated}
    Fix $\eps,\delta \in (0,1)$. Then the $(\eps,\delta)$-repetition budget for $\cA_{\text{TRR}}$ satisfies 
     \[d_{\cA_{\text{TRR}}}(\eps,\delta) \leq O\left(\log\frac{1}{\eps}\right).\] Additionally, letting $K_{j}$ denote the number of repeated comparisons that $\cA_{\text{TRR}}$ uses from context $j$,  
    \[\mathbb{E}[K_j]=O(1).\]
\end{restatable}
\begin{restatable}{corollary}{TRRsamples}\label{cor:TRRsamples}
There exists a model-pair sampling distribution $\cC$ such that the upper bound on $d_{\cA_{\mathrm{TRR}}}(\eps,\delta)$ in \cref{thm:randomtruncated} can be realized with 
\[
n = O \left(\frac{|\cM|}{\eps^2}\log\frac{|\cM|}{\delta} \right)
\]
sampled contexts. Consequently, the total number of pairwise comparisons is at most 
\[
O\left(\frac{|\cM|}{\eps^2}\log \frac{1}{\eps}\log\frac{|\cM|}{\delta}\right), \] and the expected total number of comparisons is 
    \[
    O\left(\frac{|\cM|}{\eps^2}\log \frac{|\cM|}{\delta} \right).
    \]
\end{restatable}
The proof of \cref{thm:randomtruncated} primarily revolves around controlling the bias, variance, and boundedness of the statistic
\[
\tilde{Z}_j(m) = \frac{1}{\beta}\sum_{k=1}^{K_j}\frac{\prod_{r=1}^k y_r^j - \prod_{r=1}^k (1-y_r^j)}{k\mathbb{P}(K_j \geq k)},
\]
which then allows us to apply Bernstein's Inequality to obtain concentration bounds after averaging over contexts. 

To start, we can write that 
\[
\mathbb{E}[\tilde{Z}_j(m)|i^j,t^j] = \frac{1}{\beta}\sum_{k=1}^d \frac{p_j(m)^k - (1-p_j(m))^k}{k},
\]
where $p_j(m) = \sigma(\beta(u_{i^j}(t^j,m) - u_{i^j}(t^j,m_0)))$. Note that the conditional expectation here is the same as that of the estimates involved in $\cA_{\mathrm{MLE-UNCONST}}$. However, while the MLE algorithm deterministically uses all $d$ comparisons to compute its estimate, the Russian Roulette algorithm randomizes the number of terms it computes according to $K_j$. 

Controlling the bias of $\tilde{Z}_j(m)$ amounts to bounding the higher-order (larger than $d$) terms of the series, which we show is at most 
\[
\frac{2\sigma(\beta)^{d+1}}{\beta(d+1)(1-\sigma(\beta))}.
\]
Then, choosing $d = O\left(\log \frac{1}{\eps}\right)$ guarantees that the bias is at most $\eps/8$. This is the important step that reveals the logarithmic dependence on $1/\eps$ in the repetition budget. We bound the variance by controlling the second moment of $\tilde{Z}_j(m)$ via an application of Minkowski's Inequality, showing that 
\[
\sqrt{\mathbb{E}[\tilde{Z}_j(m)^2|i^j,t^j]} = O(1).
\]
Finally, boundedness of $\tilde{Z}_j(m)$ primarily arises from the fact that $K_j$ is truncated at $d$. From here, we apply the simply aggregation procedure to empirically average these $\tilde{Z}_j(m)$ estimates over contexts, and applying Bernstein's Inequality to these empirical averages leads to the desired ranking guarantee.

\cref{thm:randomtruncated} reveals a benefit of our truncated Russian Roulette algorithm compared to the MLE-based approach --- the number of per-context repeated comparisons it actually uses is constant in expectation. This result is an artifact of randomizing the number of comparisons, $K_j$, that our estimator actually uses within every context $j$. 

There is one caveat: the truncation level involved in obtaining the truncated statistic must be set as a function of $\eps$, meaning that we must have some intuition about the minimum average utility gap prior to defining the estimator. In instances where knowing $\Delta$ a priori is infeasible, we can actually omit the truncation level entirely and consider the untruncated Russian roulette estimator that we initially presented, obtaining the same guarantees as \cref{thm:randomtruncated} and \cref{cor:TRRsamples} in expectation. We defer a more detailed presentation of the results to \cref{appendix:randomizedestimator}.

\subsection{Lower Bound on Repeated Comparisons}
The latter two MLE-based algorithms and the truncated Russian roulette algorithm succeed with a repetition budget depending logarithmically on $1/\eps$. Is there some other algorithm that could achieve a better guarantee? We present a quantitative information-theoretic lower bound  for our setting to show that the guarantees from \cref{prop:mle} and \cref{thm:randomtruncated} are tight. 
\begin{restatable}{theorem}{repeatedcomparisonslowerbound}\label{thm:repeatedcomparisonslowerbound}
    Fix $\delta \in (0,1/2)$ and $\eps \in (0,1)$. Then,
    \[
    d^*(\eps,\delta) = \Theta\left(\log\frac{1}{\eps}\right).
    \]
\end{restatable} 

The lower bound construction works by simply considering two models $A$ and $B$ and one task $t$. In particular, for a user $i$, let $p_i$ denote the BT probability that $i$ prefers $A$ to $B$. Then, as we discussed earlier, we can write the utility difference between $A$ and $B$ for this user as $u_i(t,A) - u_i(t,B) = (1/\beta)\log(p_i/(1-p_i)).$
In this setting, users can be entirely identified by their corresponding BT probability, so the average utility difference is just the expectation of the right-hand side over $p_i$. Our lower bound constructs two population distributions for $p_i$, $\mu_+$ and $\mu_-$, that are indistinguishable given only $d$ repeated comparisons but for which the average utility difference has opposite signs, indicating that the ranking of models is flipped. The key insight involved with constructing these distributions is that conditional on $p_i$, a sequence of $d$ comparisons involving $w$ wins by $A$ has probability $p_i^w(1-p_i)^{d-w}$,
which is a polynomial of degree at most $d$. Consequently, in order to ensure that $\mu_+$ and $\mu_-$ are indistinguishable from $d$ comparisons, it suffices to guarantee that they share the same first $d$ moments. However, since $\frac{1}{\beta}\log(p_i/(1-p_i))$ is not a polynomial, its expectation under $p$ sampled from $\mu_+$ versus $p$ sampled from $\mu_-$ can be different, even having opposite signs. \cref{fig:lowerbound} illustrates this idea for the $d=1$ case.

\begin{figure}[t]
    \centering

    \begin{subfigure}[t]{0.34\textwidth}
        \centering
        \begin{tikzpicture}
        \begin{axis}[
            width=0.96\linewidth,
            height=0.74\linewidth,
            xmin=0.05, xmax=0.75,
            ymin=0, ymax=1.08,
            xtick={0.3,0.4,0.5,0.6,0.7},
            ytick={0,0.2,0.4,0.6,0.8,1.0},
            xlabel={Bradley--Terry win probability \(p\)},
            ylabel={population mass},
            axis lines=left,
            tick align=outside,
            tick style={black},
            xlabel style={font=\scriptsize},
            ylabel style={font=\scriptsize},
            tick label style={font=\scriptsize},
            title={\(\mu_{+}\)},
            title style={font=\normalsize, yshift=1pt},
            clip=false,
        ]

        \addplot[
            black,
            very thick,
            ycomb,
            mark=*,
            mark size=3pt
        ]
        coordinates {
            (0.4,0.6667)
            (0.7,0.3333)
        };

        \node[font=\scriptsize] at (axis cs:0.4,0.77) {\(2/3\)};
        \node[font=\scriptsize] at (axis cs:0.7,0.43) {\(1/3\)};

        \node[align=center, font=\scriptsize] at (axis cs:0.50,1.02) {
            \(\mathbb E[p]=0.5\)\\[1pt]
            \(\mathbb E[\operatorname{logit}(p)]=+0.0121\)
        };

        \end{axis}
        \end{tikzpicture}
    \end{subfigure}
    \hspace{0.045\textwidth}
    \begin{subfigure}[t]{0.34\textwidth}
        \centering
        \begin{tikzpicture}
        \begin{axis}[
            width=0.96\linewidth,
            height=0.74\linewidth,
            xmin=0.05, xmax=0.75,
            ymin=0, ymax=1.08,
            xtick={0.3,0.4,0.5,0.6,0.7},
            ytick={0,0.2,0.4,0.6,0.8,1.0},
            xlabel={Bradley--Terry win probability \(p\)},
            ylabel={population mass},
            axis lines=left,
            tick align=outside,
            tick style={black},
            xlabel style={font=\scriptsize},
            ylabel style={font=\scriptsize},
            tick label style={font=\scriptsize},
            title={\(\mu_{-}\)},
            title style={font=\normalsize, yshift=1pt},
            clip=false,
        ]

        \addplot[
            black,
            very thick,
            ycomb,
            mark=*,
            mark size=3pt
        ]
        coordinates {
            (0.3,0.3333)
            (0.6,0.6667)
        };

        \node[font=\scriptsize] at (axis cs:0.3,0.44) {\(1/3\)};
        \node[font=\scriptsize] at (axis cs:0.6,0.77) {\(2/3\)};

        \node[align=center, font=\scriptsize] at (axis cs:0.50,1.02) {
            \(\mathbb E[p]=0.5\)\\[1pt]
            \(\mathbb E[\operatorname{logit}(p)]=-0.0121\)
        };

        \end{axis}
        \end{tikzpicture}
    \end{subfigure}

    \caption{Visual for \cref{thm:repeatedcomparisonslowerbound}
    lower bound construction in the \(d=1\) case.}
    \label{fig:lowerbound}
\end{figure}
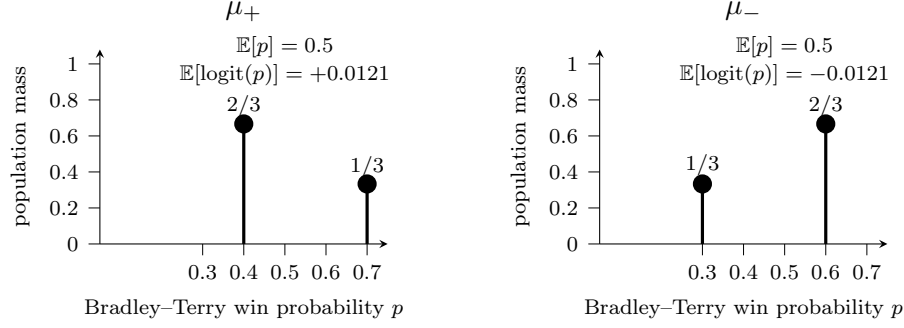

In order to obtain a quantitative guarantee, we have to control the error 
\[
E_d := \inf_{P:\text{deg}(P) \leq d} \sup_{p \in [\sigma(-\beta),1-\sigma(-\beta)]}\abs{\frac{1}{\beta}\log \frac{p}{1-p} - P(p)}
\]
corresponding to the best degree-$d$ polynomial approximation of the context-level utility difference. This is precisely the relevant quantity for our construction because we design $\mu_+$ and $\mu_-$ to have the same first $d$  moments and induce average utility gaps of $E_d$ and $-E_d$, respectively. We subsequently obtain the lower bound 

$E_d = \Omega\left(\tanh(\beta/4)^d/d \right)$
by analyzing the Chebyshev expansion of the logit function. Now, if $E_d \geq \eps$, then no algorithm can uniformly succeed on problem instances derived from $\mu_+$ and $\mu_-$ with probability greater than $1/2$. Thus, we must have $E_d < \eps$, which gives our desired lower bound for $d$.

\section{Experiments}
We run a series of synthetic and semi-synthetic Arena-based experiments comparing our four algorithms.
\paragraph{Synthetic Experiment Design.} Our first main synthetic experiment considers ranking five models with a fixed average-utility ranking and $\Delta = 0.015$. We then use a parameter $h \in [0,1]$ to interpolate between a setting where individual utility functions are homogeneous ($h=0$) and one where there is a large amount of heterogeneity $(h=1)$, without changing the target ranking or average utility levels. We examine the repetition budget of the four algorithms across a grid of $h$ values. Our second experiment seeks to examine the dependence on $1/\Delta$ predicted by the theoretical results. We fix $h=1$ and vary $\Delta$ from 0.06 to 0.002, comparing how the repetition budgets of the algorithms change as the minimum average utility gap decreases. 

\paragraph{Semi-Synthetic Experiment Design.} We construct a semi-synthetic population from the "arena-human-preference-140k" dataset \citep{chatbotarena} by clustering prompts into 30 groups. Within each group, we fit a Bradley-Terry model to the votes associated with its prompts, resulting in scores for each LLM varying across prompt clusters. We then select five LLMs and use these fitted scores to construct utility profiles. Similar to the first experiment, we keep a fixed average utility ranking and $\Delta=0.005$. Now, rather than interpolating between different heterogeneity levels, we interpolate between a benign task distribution and a hard distribution optimized against $\cA_{\mathrm{MLE-AVG}}$. The benign distribution is constructed simply by taking the empirical weights of the prompt clusters from the dataset. And importantly, the reweighting in the hard distribution preserves every model's population-average utility. Using an interpolation parameter $\lambda \in [0,1]$, we examine how the repetition budget of the four algorithms varies across different values of $\lambda$.

\begin{figure*}[!t]
  \centering
  \resizebox{0.8\textwidth}{!}{%
    \begin{tikzpicture}
\begin{groupplot}[
  paperaxis,
  group style={group size=2 by 1,horizontal sep=1.0cm},
  width=6.9cm,height=4.6cm,
  xmin=-0.035,xmax=1.035,ymax=800,
  xtick={0,.2,.4,.6,.8,1},
  xticklabels={$0$,$0.2$,$0.4$,$0.6$,$0.8$,$1$},
  ytick={1,10,100},
  ylabel={Repeated comparisons per context},
]
\nextgroupplot[xlabel={Heterogeneity $h$}]
\addplot[avgline] coordinates {
  (0.0,1.0)
  (0.1,1.0)
  (0.2,1.0)
  (0.3,1.0)
  (0.4,1.0)
  (0.5,1.0)
  (0.6,1.0)
  (0.7,1.0)
  (0.8,384.0)
  (0.9,384.0)
  (1.0,512.0)
};
\addplot[calline] coordinates {
  (0.0,12.0)
  (0.1,12.0)
  (0.2,12.0)
  (0.3,12.0)
  (0.4,12.0)
  (0.5,12.0)
  (0.6,12.0)
  (0.7,12.0)
  (0.8,12.0)
  (0.9,12.0)
  (1.0,12.0)
};
\addplot[unconstline] coordinates {
  (0.0,1.0)
  (0.1,1.0)
  (0.2,1.0)
  (0.3,1.0)
  (0.4,1.0)
  (0.5,1.0)
  (0.6,1.0)
  (0.7,1.0)
  (0.8,4.0)
  (0.9,4.0)
  (1.0,4.0)
};
\addplot[trrexpline] coordinates {
  (0.0,1.0)
  (0.1,1.0)
  (0.2,1.0)
  (0.3,1.0)
  (0.4,1.0)
  (0.5,1.0)
  (0.6,1.0)
  (0.7,1.0)
  (0.8,1.8445791535340528)
  (0.9,1.8445791535340528)
  (1.0,1.8445791535340528)
};
\addplot[trrmaxline] coordinates {
  (0.0,1.0)
  (0.1,1.0)
  (0.2,1.0)
  (0.3,1.0)
  (0.4,1.0)
  (0.5,1.0)
  (0.6,1.0)
  (0.7,1.0)
  (0.8,4.0)
  (0.9,4.0)
  (1.0,4.0)
};

 \node[anchor=north west,font=\bfseries\fontsize{10}{11}\selectfont]
   at (rel axis cs:0.035,0.97) { Synthetic};
\nextgroupplot[xlabel={Task Distribution Mixture Weight $\lambda$},ylabel={}]
\addplot[avgline] coordinates {
  (0.0,1.0)
  (0.1,1.0)
  (0.2,1.0)
  (0.3,1.0)
  (0.4,1.0)
  (0.5,1.0)
  (0.6,1.0)
  (0.7,1.0)
  (0.8,1.0)
  (0.9,96.0)
  (1.0,96.0)
};
\addplot[calline] coordinates {
  (0.0,16.0)
  (0.1,16.0)
  (0.2,16.0)
  (0.3,16.0)
  (0.4,16.0)
  (0.5,16.0)
  (0.6,16.0)
  (0.7,16.0)
  (0.8,16.0)
  (0.9,16.0)
  (1.0,16.0)
};
\addplot[unconstline] coordinates {
  (0.0,1.0)
  (0.1,1.0)
  (0.2,1.0)
  (0.3,1.0)
  (0.4,1.0)
  (0.5,1.0)
  (0.6,1.0)
  (0.7,1.0)
  (0.8,1.0)
  (0.9,4.0)
  (1.0,4.0)
};
\addplot[trrexpline] coordinates {
  (0.0,1.0)
  (0.1,1.0)
  (0.2,1.0)
  (0.3,1.0)
  (0.4,1.0)
  (0.5,1.0)
  (0.6,1.0)
  (0.7,1.0)
  (0.8,1.0)
  (0.9,1.8445791535340528)
  (1.0,1.8445791535340528)
};
\addplot[trrmaxline] coordinates {
  (0.0,1.0)
  (0.1,1.0)
  (0.2,1.0)
  (0.3,1.0)
  (0.4,1.0)
  (0.5,1.0)
  (0.6,1.0)
  (0.7,1.0)
  (0.8,1.0)
  (0.9,4.0)
  (1.0,4.0)
};

 \node[anchor=north west,font=\bfseries\fontsize{10}{11}\selectfont]
   at (rel axis cs:0.035,0.97) {Arena};
\end{groupplot}
\node[anchor=south,inner sep=0pt] at
 ($(group c1r1.north west)!0.5!(group c2r1.north east)+(0,0.28cm)$) {\paperlegend};
\end{tikzpicture}%
  }
  \caption{\textbf{Left:} Synthetic populations with varying heterogeneity $h$ and fixed $\Delta = 0.015$. \textbf{Right:} Arena-based populations interpolating between empirical and optimized task distributions, with fixed $\Delta = 0.005$. Population-average utilities remain fixed within each graph. We distinguish between the expected and maximum number of repeated comparisons of $\cA_{\mathrm{TRR}}$ via the solid and dotted lines, respectively.}
  \label{fig:repetition-budget}
\end{figure*}
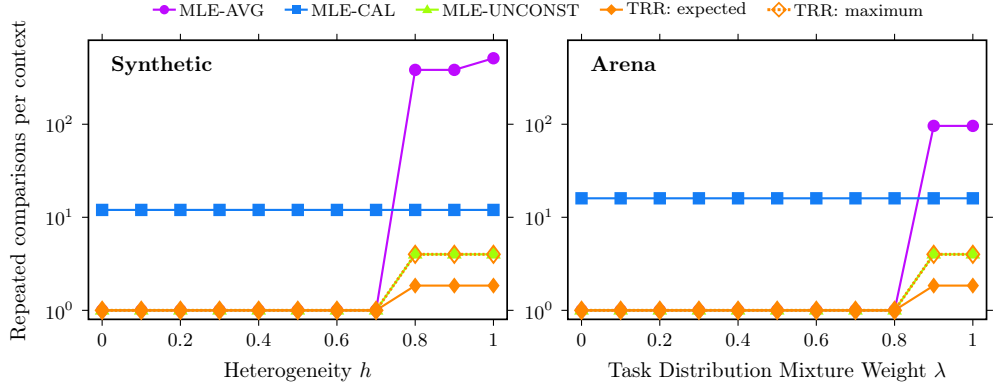

\paragraph{Results.} The graphs in \cref{fig:repetition-budget} indicate that the repetition budget of $\cA_{\mathrm{MLE-AVG}}$ can significantly degrade under harder problem instances, while the other three algorithms are robust to changes in heterogeneity levels or the underlying task distribution. \cref{fig:change-delta} concerns the second synthetic experiment and reveals that, as expected, the repetition budget of $\cA_{\mathrm{MLE-AVG}}$ has a much sharper dependence on $1/\Delta$ compared to the other three algorithms. Additionally, the expected number of comparisons of $\cA_{\mathrm{TRR}}$ only varies between 1 and 2.24 across values of $1/\Delta$. For a detailed description of the methodology and results, we refer the reader to \cref{appendix:experiments}.

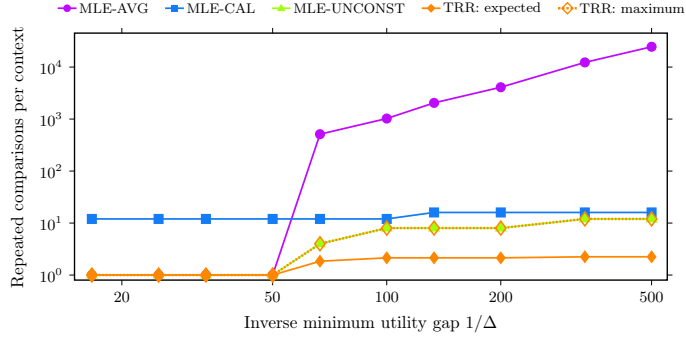
\begin{figure*}[!t]
  \centering
  \resizebox{0.55\textwidth}{!}{%
    \begin{tikzpicture}
\begin{groupplot}[
  group style={group name=gapgroup,group size=1 by 1},
  paperaxis,
  width=12.0cm,height=5.0cm,
  xmode=log,log basis x=10,
  xmin=15,xmax=550,ymax=45000,
  xtick={20,50,100,200,500},
  xticklabels={$20$,$50$,$100$,$200$,$500$},
  ytick={1,10,100,1000,10000},
  xlabel={Inverse minimum utility gap $1/\Delta$},
  ylabel={Repeated comparisons per context},
]
\nextgroupplot
\addplot[avgline] coordinates {
  (16.666666666666668,1.0)
  (25.0,1.0)
  (33.333333333333336,1.0)
  (50.0,1.0)
  (66.66666666666667,512.0)
  (100.0,1024.0)
  (133.33333333333334,2048.0)
  (200.0,4096.0)
  (333.3333333333333,12288.0)
  (500.0,24576.0)
};
\addplot[calline] coordinates {
  (16.666666666666668,12.0)
  (25.0,12.0)
  (33.333333333333336,12.0)
  (50.0,12.0)
  (66.66666666666667,12.0)
  (100.0,12.0)
  (133.33333333333334,16.0)
  (200.0,16.0)
  (333.3333333333333,16.0)
  (500.0,16.0)
};
\addplot[unconstline] coordinates {
  (16.666666666666668,1.0)
  (25.0,1.0)
  (33.333333333333336,1.0)
  (50.0,1.0)
  (66.66666666666667,4.0)
  (100.0,8.0)
  (133.33333333333334,8.0)
  (200.0,8.0)
  (333.3333333333333,12.0)
  (500.0,12.0)
};
\addplot[trrexpline] coordinates {
  (16.666666666666668,1.0)
  (25.0,1.0)
  (33.333333333333336,1.0)
  (50.0,1.0)
  (66.66666666666667,1.8445791535340528)
  (100.0,2.143326669400975)
  (133.33333333333334,2.143326669400975)
  (200.0,2.143326669400975)
  (333.3333333333333,2.24405102272833)
  (500.0,2.24405102272833)
};
\addplot[trrmaxline] coordinates {
  (16.666666666666668,1.0)
  (25.0,1.0)
  (33.333333333333336,1.0)
  (50.0,1.0)
  (66.66666666666667,4.0)
  (100.0,8.0)
  (133.33333333333334,8.0)
  (200.0,8.0)
  (333.3333333333333,12.0)
  (500.0,12.0)
};

\end{groupplot}
\node[anchor=south,inner sep=0pt] at
 ($(gapgroup c1r1.north west)!0.5!(gapgroup c1r1.north east)+(0,0.28cm)$) {\paperlegend};
\end{tikzpicture}%
  }
  \caption{Synthetic populations with varying $\Delta$ and fixed $h = 1$ and population-average utility ranking.  We distinguish between the expected and maximum number of repeated comparisons of $\cA_{\mathrm{TRR}}$ via the solid and dotted line, respectively.}
  \label{fig:change-delta}
\end{figure*}
\section{Acknowledgments}
Han Shao acknowledges support from an Adobe Research gift.
\newpage
\bibliographystyle{plainnat}
\bibliography{ref}
\newpage
\appendix
\newpage
\section{Additional Related Works}
Classical work on ranking from pairwise comparisons primarily studies how the total number and topology of comparisons affect estimation under a common latent score vector. \citet{shah2015} derive sharp minimax bounds for pairwise-comparison models such as Bradley–Terry and characterize how estimation error depends on the comparison graph. \citet{shah2018} study exact ranking and top-\(k\) recovery under more general pairwise-comparison matrices. In these settings, however, there is a shared underlying score vector. Our setting allows latent rankings to vary across contexts arbitrarily rather than treating variation in pairwise comparisons as noise surrounding a global ranking.

When considering heterogeneity of users, some previous works address the issue by instead reconsidering the metric that we seek to estimate or report. \citet{frick2025prompttoleaderboard} contrasts with LM Arena, outputting multiple leaderboards, each catered to a single prompt or collection of prompts, rather than a single global ranking. \citet{haghtalab2026pluralistic} employ ideas from social choice theory to output a leaderboard that remains stable with respect to heterogeneous user preferences. Our approach is complementary to these --- we fix average utility as the target of choice, remaining agnostic about its efficacy as an appropriate aggregation criterion. Instead, we study its sample complexity of ranking recovery under unrestricted user utility functions. In contrast, \cite{aouad2026} study a similar problem involving recovering average utility but place symmetry assumptions on the users' heterogeneous utility functions.

In the alignment setting, the high cost of obtaining pairwise comparisons has motivated approaches which use dueling bandits to model the alignment problem and employ active learning approaches to reduce the number of required comparisons \citep{ji2025reinforcement,mehta2025sample}. 
\newpage
\section{General Lemmas}
Below we state a lemma that is used for multiple proofs in the appendix. 
\begin{lemma} \label{lem:controlobs}
Fix $r \geq 1$, $\delta \in (0,1)$, and a dataset $S$ generated by parameters $(n,d)$. Let $\cC = \Unif(\{\{m,m_0\}:m \in \cM \setminus m_0\})$. Then, if we obtain
\[
n \geq \max\left\{2(|\cM|-1)r\;,\;8(|\cM|-1)\log\frac{|\cM|-1}{\delta} \right\}
\]
sampled contexts, we have that 
\[
\mathbb{P}(|\{j \in [n]: \{m_1^j,m_2^j\} = \{m,m_0\}\}| \geq r) \geq 1-\delta
\]
simultaneously for all $m \neq m_0$.
\end{lemma}
\begin{proof}
Fix $m \neq m_0$. Let 
\[
I_m := \{j \in [n]: \{m_1^j,m_2^j\} = \{m,m_0\}\}
\]
Then by definition, $|I_m| \sim \text{Binomial}\left(n,\frac{1}{|\cM|-1}\right)$. Then, if we take $n \geq 2(|\cM|-1)r$, we have that 
\[
\frac{1}{2}\mathbb{E}[|I_m|] = \frac{1}{2}\frac{n}{|\cM|-1} \geq r.
\]
Applying a multiplicative Chernoff bound gives,
\[
\mathbb{P}(|I_m| < r) \leq \mathbb{P}\left(|I_m| < \frac{\mathbb{E}[|I_m|]}{2}\right) \leq \exp\left(-\frac{\mathbb{E}[|I_m|]}{8}\right) = \exp\left(-\frac{n}{8(|\cM|-1)}\right)
\]
Taking a union bound over the $|\cM|-1$ non-reference models gives that simultaneously for all $m \neq m_0$,
\[
\mathbb{P}(|I_m| < r) \leq (|\cM|-1)\exp\left(-\frac{n}{8(|\cM|-1)}\right)
\]
Taking $n \geq 8(|\cM|-1)\log\frac{|\cM|-1}{\delta}$ guarantees that the right hand side is at most $\delta$, concluding the proof.
\end{proof}
\newpage

\section{MLE-Based Algorithms}\label{appendix:mle}
We start by presenting the common setup of each of the MLE-based algorithms. 
Fix an arbitrary reference model $m_0 \in \mathcal{M}$. It suffices to estimate $\bar{u}(m) - \bar{u}(m_0)$. Given a dataset $S$ with parameters $(n,d)$, define the negative log-likelihood
\[
\begin{aligned}
L_S(u)
= -\sum_{j=1}^n\sum_{r=1}^d \Bigl[
y_r^j\log\sigma\!\left(
\beta\bigl(
u(i^j,t^j,m_1^j)-u(i^j,t^j,m_2^j)
\bigr)
\right)  +(1-y_r^j)\log\sigma\!\left(
\beta\bigl(
u(i^j,t^j,m_2^j)-u(i^j,t^j,m_1^j)
\bigr)
\right)
\Bigr].
\end{aligned}
\]
Now, for each $m \neq m_0$, let 
\[
I_m := \{j \in [n]: \{m_1^j,m_2^j\} = \{m,m_0\}\}
\]
denote the set of sampled contexts in which $m$ is compared to $m_0$. 

\subsection{Formal Definition of $\cA_{\mathrm{MLE-AVG}}$} \label{mle-avg-defn}
\paragraph{Estimator.} We consider the constrained MLE estimate:
\[
\widehat{v} \in \argmin_{\substack{v:\cU\times\cT\times\cM\to[-1,1] \\v(i,t,m_0)=0, \forall (i,t)}} L_{S}(v).
\]
We choose to use $\widehat{v}$ instead of $\widehat{u}$ to denote the estimate here to avoid confusion, since we are now estimating utility \textit{differences} relative to a reference model instead of utilities directly. 

\paragraph{Aggregation.} Now, we estimate the average utility relative to $m_0$ by taking the empirical average over contexts. In particular, for $m \neq m_0$
\[
\widehat{g}_{\mathrm{MLE-AVG}}(m) = \frac{1}{|I_m|} \sum_{j \in I_m} \hat{v}(i^j,t^j,m)
\]
for each $m \neq m_0$. Define $\widehat{g}_{\mathrm{MLE-AVG}}(m_0) = 0$ and for completeness, if $I_m$ is empty, define $\widehat{g}_{\mathrm{MLE-AVG}}(m) = 0$. The corresponding ranking algorithm, $\cA_{\mathrm{MLE-AVG}}$ ranks models in decreasing order of $\widehat{g}_{\mathrm{MLE-AVG}}(m)$, breaking ties according to the fixed ordering of $\cM$.

\subsection{Formal Definition of $\cA_{\mathrm{MLE-CAL}}$} \label{mle-cal-defn}

\paragraph{Estimator.} We again consider the constrained MLE estimate:
\[
\widehat{v} \in \argmin_{\substack{v:\cU\times\cT\times\cM\to[-1,1] \\v(i,t,m_0)=0, \forall (i,t)}} L_{S}(v).
\]

\paragraph{Aggregation.} To account for the bias that comes from directly averaging the MLE estimates, we calibrate them first. In particular, for fixed $\epsilon \in (0,1)$, define $\phi_d: [-1,1]\times(0,1) \ra \mathbb{R}$ to be a solution to 
\[
\begin{aligned}
\phi_d \in \arg\min_{\phi:[-1,1]\times(0,1)\to\mathbb{R}}
\quad&
\sup_{x\in[-1,1]} |\phi(x,\eps)| \\
\text{subject to}\quad&
\sup_{D\in[-1,1]}
\left|
\mathbb{E}\left[
\phi\left(\widehat v(i^j,t^j,m),\eps\right)
\middle|
u_{i^j}(t^j,m)-u_{i^j}(t^j,m_0)=D
\right]
-D
\right|
\le
\frac{\varepsilon}{8}.
\end{aligned}
\]

Note that for fixed $d$, computing $\phi_d$ actually reduces to an offline finite-dimensional convex optimization problem because the MLE estimates only take finitely many values. We discuss the implementation in detail in \cref{appendix:experiments}.

Now that we have defined the calibration functions, we can take the empirical average over $I_m$ for each $m \neq m_0$: 
\[
\widehat{g}_{\mathrm{MLE-CAL}}(m) = \frac{1}{|I_m|} \sum_{j \in I_m} \phi_d(\hat{v}(i^j,t^j,m),\eps)
\]
for each $m \neq m_0$. Define $\widehat{g}_{\mathrm{MLE-CAL}}(m_0) = 0$ and for completeness, if $I_m$ is empty, define $\widehat{g}_{\mathrm{MLE-CAL}}(m) = 0$. The corresponding ranking algorithm, $\cA_{\mathrm{MLE-CAL}}$ ranks models in decreasing order of $\widehat{g}_{\mathrm{MLE-CAL}}(m)$, breaking ties according to the fixed ordering of $\cM$.

\subsection{Formal Definition of $\cA_{\mathrm{MLE-UNCONST}}$} \label{mle-unconst-defn}
\paragraph{Estimator.} We consider the unconstrained MLE estimate 
\[
\widehat{v} \in \argmin_{{v:v(i,t,m_0)=0, \forall (i,t)}} L_{S}(v).
\]
\paragraph{Aggregation.}
In this setting, for any $m \neq m_0$ and $j \in I_m$, 
\[
\widehat{v}(i^j,t^j,m) = \frac{1}{\beta}\log \frac{W_j(m)}{d-W_j(m)},
\]
where $W_j(m)$ denotes the number of times $m$ wins in the context. Then, we can easily extract $W_j(m)$ from our estimate 
We apply a harmonic transformation prior to averaging over contexts in $I_m$. This gives the utility difference estimate 
\[
\widehat{g}_{\mathrm{MLE-UNCONST}}(m) = \frac{1}{\beta|I_m|}\sum_{j \in I_m}(H_{W_j(m)}-H_{d-W_j(m)}).
\]
Define $\widehat{g}_{\mathrm{MLE-UNCONST}}(m_0) = 0$ and for completeness, if $I_m$ is empty, define $\widehat{g}_{\mathrm{MLE-UNCONST}}(m) = 0$. The corresponding ranking algorithm, $\cA_{\mathrm{MLE-UNCONST}}$ ranks models in decreasing order of $\widehat{g}_{\mathrm{MLE-UNCONST}}(m)$, breaking ties according to the fixed ordering of $\cM$.
\subsection{Proof of \cref{prop:mleavg}}
\mleavg*
\begin{proof}
It suffices to consider two models, $\cM = \{m,m_0\}$ and one task $t$, so the task distribution is degenerate and there is only one possible model-pair sampling distribution $\cC$. For a user $i$, let $W_i \sim \text{Binomial}(d,\sigma(\beta(u_i(t,m) - u_i(t,m_0))))$ denote the number of wins for model $m$ under $d$ comparisons by user $i$. The constrained MLE has a closed form representation
\[
\hat{v}(i,t,m) = \frac{1}{\beta}\log \frac{\min \{1-\sigma(-\beta), \max\{\sigma(-\beta),W_i/d\}\}}{1-\min \{1-\sigma(-\beta), \max\{\sigma(-\beta),W_i/d\}\}}
\]
Now, we derive two bounds relating to $\mathbb{E}[\hat{v}(i,t,m)|u_i(t,m)-u_i(t,m_0)]$. For the first bound, consider $u_i(t,m)-u_i(t,m_0) = 1$. By a standard application of the Berry-Esseen Theorem \citep{berry,esseen} to Binomial random variables, there exist constants $c_1$ and $c_2$, depending only on $\beta$, such that for sufficiently large $d$
\[
\mathbb{P}\left(\frac{W_i}{d} \leq \sigma(\beta) - \frac{c_1}{\sqrt{d}} \mid u_i(t,m)-u_i(t,m_0)=1 \right) \geq c_2
\]
Now, when the true utility difference is 1, the corresponding Bradley-Terry win probability is just $\sigma(\beta)$. On this event, the truncated empirical win probability satisfies, \[\hat{p}_i \leq \sigma(\beta) - \frac{c_1}{\sqrt{d}},\] where $\hat{p}_i = \min \{1-\sigma(-\beta), \max\{\sigma(-\beta),W_i/d\}\}$. 
\[
\begin{aligned}
1-\hat{v}(i,t,m) &= \frac{1}{\beta}\log \frac{\sigma(\beta)}{1-\sigma(\beta)} - \frac{1}{\beta}\log \frac{\hat{p}_i}{1-\hat{p}_i} \\
&\geq \frac{4}{\beta}(\sigma(\beta) - \hat{p}_i) \\
& \geq \frac{4c_1}{\beta \sqrt{d}},
\end{aligned}
\]
where the first inequality follows by applying the Mean-Value Theorem to $h(p) = \frac{1}{\beta}\log \frac{p}{1-p}$ and using $p(1-p) \leq (1/4)$ for every $p \in (0,1)$. This gives the bias bound discussed in the main paper. 

Now, fix $u_i(t,m) - u_i(t,m_0) = 1/2$. The win probability for model $m$ here is $\sigma(\beta/2)$ which is a constant distance away from both truncation boundaries $\sigma(\beta)$ and $\sigma(-\beta)$ in the constrained MLE. Thus, Hoeffding's Inequality implies that on this event, the probability that $W_i/d$ reaches either of these boundaries decays exponentially with $d$. Now, if truncation does not occur, then the MLE estimate is simply $h(W_i/d)$. A second-order Taylor expansion centered around $p_i = \sigma(\beta/2)$ then gives
\[
h(W_i/d) = h(p_i) + h'(p_i)\left(\frac{W_i}{d}-p_i\right) + O\left(\left(\frac{W_i}{d}-p_i \right)^2\right).
\]
Since $\mathbb{E}[W_i/d - p_i] = 0$ and \[\mathbb{E}\left[\left(\frac{W_i}{d}-p_i \right)^2\right] = \frac{p_i(1-p_i)}{d} = O(1/d),\]
we have that 
\[
\abs{\mathbb{E}[\hat{v}(i,t,m)|u_i(t,m)-u_i(t,m_0)=1/2]-1/2} \leq \frac{C_{\beta}}{d}
\]
for some constant $C_{\beta} > 0$.

We now construct a problem instance where these two bias bounds cause $\cA_{\mathrm{MLE-AVG}}$ to output the incorrect ranking for any finite number of contexts. Choose a sufficiently small constant $\eta>0$,
depending only on $\beta$ and $\delta$, and suppose that a fraction
\[
q_+ = \frac13+\frac{\eta}{\sqrt d}
\]
of users have utility difference $1$, while the remaining users have
utility difference $-1/2$. The average utility difference between $m$
and $m_0$ is then
\[
\bar{u}(m)-\bar{u}(m_0)
=
q_+-\frac{1-q_+}{2}
=
\frac{3\eta}{2\sqrt d}.
\]
Thus, $m$ is ranked above $m_0$.

On the other hand, using the two bias bounds above,
\[
\begin{aligned}
\mathbb E[\widehat v(i,t,m)]
&\leq
q_+\left(1-\frac{c_\beta}{\sqrt d}\right)
+
(1-q_+)\left(-\frac12+\frac{C_\beta}{d}\right)\\
&=
\frac{3\eta}{2\sqrt d}
-
q_+\frac{c_\beta}{\sqrt d}
+
(1-q_+)\frac{C_\beta}{d}.
\end{aligned}
\]
Take $\eta\leq c_\beta/12$. Since $q_+\geq 1/3$, for sufficiently large
$d$ the final term is at most $c_\beta/(12\sqrt d)$, and therefore
\[
\mathbb E[\widehat v(i,t,m)]
\leq
-\frac{c_\beta}{8\sqrt d}.
\]
Hence, although $m$ has larger average utility, the quantity averaged by
$\cA_{\mathrm{MLE-AVG}}$ has negative expectation.

We next show that increasing the number of sampled contexts cannot resolve
this issue. Since every constrained MLE lies in $[-1,1]$, Hoeffding's
Inequality gives
\[
\Pr\left(
\widehat g_{\mathrm{MLE-AVG}}(m)\geq 0
\right)
\leq
\exp\left(
-\frac{c_\beta^2 n}{128d}
\right).
\]
Therefore, whenever $n\geq C_{\beta,\delta}d$ for a sufficiently large
constant $C_{\beta,\delta}$, the algorithm outputs the incorrect ranking
with probability greater than $\delta$.

It remains to rule out using fewer contexts. Consider a second population
in which the fraction of users with utility difference $1$ is
\[
q_-=\frac13-\frac{\eta}{\sqrt d},
\]
with the remaining users again having utility difference $-1/2$. Its
average utility difference is
\[
u(m)-u(m_0)
=
-\frac{3\eta}{2\sqrt d},
\]
so the two populations have opposite rankings and the same magnitude of
the average utility gap.

Even if we reveal which of the two user types was sampled in each context,
distinguishing the two populations from $n$ contexts amounts to
distinguishing $n$ Bernoulli draws with parameters $q_+$ and $q_-$. Since
$q_+$ and $q_-$ are bounded away from $0$ and $1$,
\[
D_{\mathrm{KL}}
\left(
\operatorname{Bernoulli}(q_+)
\,\middle\|\,
\operatorname{Bernoulli}(q_-)
\right)
=
O_\beta\left(\frac{\eta^2}{d}\right).
\]
Thus, for $n$ independently sampled contexts, the KL divergence is
$O_\beta(n\eta^2/d)$. Pinsker's Inequality and the standard two-point
testing bound imply that, whenever
\[
n
\leq
c_{\beta,\delta}\frac{d}{\eta^2},
\]
at least one of the two populations has ranking error greater than
$\delta$.

We may choose $\eta$ sufficiently small so that
\[
c_{\beta,\delta}\frac{d}{\eta^2}
\geq
C_{\beta,\delta}d.
\]
Hence, the small-$n$ and large-$n$ cases together cover every finite
$n$.

Finally, both populations have average utility gap
\[
\frac{3\eta}{2\sqrt d}.
\]
Therefore, whenever
\[
d
\leq
\frac{9\eta^2}{4\eps^2},
\]
their average utility gap is at least $\eps$, yet no finite choice of
$n$ gives the guarantee required by Definition~2.1. It follows that
\[
d_{\cA_{\mathrm{MLE-AVG}}}(\eps,\delta)
=
\Omega_{\beta,\delta}\left(\frac{1}{\eps^2}\right).
\]

The argument above was written for sufficiently large $d$. The finitely
many remaining values of $d$ can be absorbed by decreasing the constant
in the final lower bound. For any fixed $d$, the expected constrained MLE
is a degree-$d$ polynomial in the Bradley--Terry win probability, whereas
the true utility difference is its log-odds. Thus, the two cannot agree up
to a common multiplicative constant over all contexts, and the same
two-population testing lower bound argument gives a fixed positive utility gap on which
direct MLE averaging has the wrong sign.
\end{proof}

\subsection{Proof of \cref{prop:mle}}
\cref{prop:mle} gives results for both $\cA_{\mathrm{MLE-CAL}}$ and $\cA_{\mathrm{MLE-UNCONST}}$. We separately prove the repetition budget upper bounds in \cref{prop:mle-cal} and \cref{prop:mle-unconst}, respectively, and the theorem proof simply involves combining them.
\begin{lemma}\label{lem:defq}
There exists a constant $c_{\beta} > 0$ such that for every odd $d \geq 1$, there exists a degree-$d$ polynomial $Q_d(p)$ with Bernstein representation
\[
Q_d(p) = \sum_{w=0}^d b_w\binom{d}{w}p^w(1-p)^{d-w}
\]
satisfying
\[
\sup_{p \in [\sigma(-\beta),1-\sigma(-\beta)]}\abs{Q_d(p)-\frac{1}{\beta}\log \frac{p}{1-p}} \leq \exp\{-c_{\beta}d\}.
\]
Also, the Bernstein coefficients of $Q_d(p)$ satisfy 
\[
b_0 = \ldots = b_{\ell}, \qquad b_{d-\ell}=\ldots=b_{d},
\]
where $\ell = \lfloor \sigma(-\beta)d \rfloor$.
\end{lemma}
\begin{proof}
Let $h(p) = \frac{1}{\beta}\log \frac{p}{1-p}$. Our goal is to construct a degree-$d$ polynomial that approximates $h$ while also respecting the truncation of the constrained MLE. To start, we impose that $Q_d(1/2)=h(1/2) = 0$, so $Q_d$ shares the same zero as $h$. Then, we have that 
\[
Q_d(p) - h(p) = \int_{1/2}^p (Q'_d(q)-h'(q))dq.
\]
Hence, if the derivative of $Q_d$ approximates the derivative of $h$ well, integrating gives that $Q_d$ approximates $h$ well. This motivates defining $Q_d$ through its derivative. We can differentiate the Bernstein representation of $Q_d$ to obtain
\[
Q'_d(p) = d\sum_{w=0}^{d-1}
(b_{w+1}-b_w)
\binom{d-1}{w}
p^w(1-p)^{d-1-w}.
\]
Now, observe that because of the truncation of the constrained MLE, the MLE outputs corresponding to $0,...,\ell$ wins are all $-1$ and the outputs corresponding to $d-\ell,...,d$ wins are all $1$. This explains the desired conditions $b_0=...=b_{\ell}$ and $b_{d-\ell}=...=b_d$. A natural way to enforce this is by requiring that $Q_d'(p)$ contain the factor $p^{\ell}(1-p)^{\ell}$. Now, we want $Q_d'$ to also approximate 
\[
h'(p) = \frac{1}{\beta p(1-p)} = \frac{4}{\beta(1-(2p-1)^2)}
\]
Now, we can apply the power series identity
\[
\frac{1}{(1-z)^{\ell+1}}
=
\sum_{r=0}^{\infty}
\binom{\ell+r}{r}z^r
\]
to obtain
\[
h'(p)
=
\frac{4}{\beta}
\left(1-(2p-1)^2\right)^\ell
\sum_{r=0}^{\infty}
\binom{\ell+r}{r}(2p-1)^{2r} = \frac{4^{\ell+1}}{\beta}p^{\ell}(1-p)^{\ell}\sum_{r=0}^{\infty}
\binom{\ell+r}{r}(2p-1)^{2r}.
\]
Conveniently, we already have the factor $p^{\ell}(1-p)^{\ell}$. Now, since we want $Q_d$ to have degree at most $d$, we write its derivative as a truncation of this series, obtaining
\[
Q_d'(p)
=
\frac{4}{\beta}
\left(1-(2p-1)^2\right)^\ell
\sum_{r=0}^{(d-1)/2-\ell}
\binom{\ell+r}{r}(2p-1)^{2r}.
\]
Since $\sigma(-\beta)<1/2$ and $d$ is odd, $\ell\leq(d-1)/2$. Moreover, the degree of $Q_d'$ is at most
\[
2\ell
+
2\left(\frac{d-1}{2}-\ell\right)
=
d-1,
\]
so $Q_d$ has degree at most $d$, as needed. It remains to show that the approximation error decays exponentially with $d$. To start, we write
\[
\frac{Q_d'(p)}{h'(p)}
=
(1-(2p-1)^2)^{\ell+1}
\sum_{r=0}^{(d-1)/2-\ell}
\binom{\ell+r}{r}(2p-1)^{2r}.
\]
Using the standard equivalence between the negative binomial and binomial CDFs gives
\[
\frac{Q_d'(p)}{h'(p)}
=
\mathbb{P}\left(
X\leq\frac{d-1}{2}-\ell
\right),
\]
where
\[
X\sim
\operatorname{Binomial}\left(
\frac{d+1}{2},(2p-1)^2
\right).
\]
Hence,
\[
1-\frac{Q_d'(p)}{h'(p)}
=
\mathbb{P}\left(
X\geq\frac{d+1}{2}-\ell
\right).
\]

For every $p\in[\sigma(-\beta),1-\sigma(-\beta)]$,
\[
(2p-1)^2\leq(1-2\sigma(-\beta))^2.
\]
On the other hand,
\[
\frac{\frac{d+1}{2}-\ell}{(d+1)/2}
=
1-\frac{2\ell}{d+1}
\geq
1-2\sigma(-\beta).
\]
Since
\[
(1-2\sigma(-\beta))^2<1-2\sigma(-\beta),
\]
we have that 
\[
\frac{\frac{d+1}{2}-\ell}{(d+1)/2}-(2p-1)^2
\ge
(1-2\sigma(-\beta))-(1-2\sigma(-\beta))^2
=
2\sigma(-\beta)(1-2\sigma(-\beta))>0.
\]
Thus, Hoeffding's Inequality gives that,
\[
0
\leq
1-\frac{Q_d'(p)}{h'(p)}
\leq
\exp(-c_\beta d)
\]
uniformly over $p\in[\sigma(-\beta),1-\sigma(-\beta)]$, where $c_{\beta} = (2\sigma(-\beta)(1-2\sigma(-\beta)))^2$.

Finally, since
$Q_d(1/2)=h(1/2)=0$, we have that
\[
\begin{aligned}
|Q_d(p)-h(p)|
&=
\left|
\int_{1/2}^p
(Q_d'(q)-h'(q))dq
\right|\\
&\leq
\exp(-c_\beta d)
\left|
\int_{1/2}^p
h'(q)dq
\right|\\
&=
\exp(-c_\beta d)|h(p)|\\
&\leq
\exp(-c_\beta d),
\end{aligned}
\]
where the final inequality follows since $h(p)\in[-1,1]$ for $p\in[\sigma(-\beta),1-\sigma(-\beta)]$.

\end{proof}
\begin{lemma}\label{lem:boundbernstein}
Let 
\[
Q_d(p) = \sum_{w=0}^d b_w\binom{d}{w}p^w(1-p)^{d-w}
\]
be the polynomial constructed in \cref{lem:defq}. There exists a constant $C_{\beta} > 0$ such that 
\[
max_{w \in \{0,\ldots,d\}}|b_w| \leq C_{\beta}5^d.
\]
\end{lemma}
\begin{proof}
Start by writing $Q_d$ in its monomial representation:
\[
Q_d(p)
=
\sum_{r=0}^d q_rp^r.
\]
We first bound the sum of the absolute values of the coefficients $q_r$.

Recall from the construction in \cref{lem:defq} that
\[
Q_d'(p)
=
\frac{4}{\beta}
(4p(1-p))^\ell
\sum_{r=0}^{(d-1)/2-\ell}
\binom{\ell+r}{r}(2p-1)^{2r}.
\]
The sum of the absolute values of the coefficients of $4p(1-p)$ is $8$, so the corresponding sum for
\[
(4p(1-p))^\ell
\]
is at most $8^\ell$. Similarly, the sum of the absolute values of the coefficients of $(2p-1)^{2r}$ is
\[
3^{2r}=9^r.
\]
It follows that the sum of the absolute values of the monomial coefficients of $Q_d'$ is at most
\[
\frac{4}{\beta}
8^\ell
\sum_{r=0}^{(d-1)/2-\ell}
\binom{\ell+r}{r}9^r.
\]
Using $\binom{\ell+r}{r}
\leq
2^{\ell+r},$
the above expression is bounded by
\[
\frac{
16^\ell}{\beta}
\sum_{r=0}^{(d-1)/2-\ell}
18^r
\leq
\frac{
16^\ell}{\beta}
\frac{1}{\beta}18^{(d-1)/2-\ell} \leq \frac{1}{\beta}18^{(d-1)/2-\ell} \leq \frac{1}{\beta}5^d.
\]

Integrating $Q_d'$ cannot increase the sum of the absolute values of the coefficients in front of $p,...,p^d$. The additive constant $C$ is chosen so that $Q_d(1/2)=0$. Then, if we write
\[
Q'_d(p) = \sum_{r=0}^{d-1}a_rp^r
\]
we have that
\[
Q_d(p) = C + \sum_{r=0}^{d-1}\frac{a_r}{r+1}p^{r+1} = 0
\]
Rearranging, using the fact that $\frac{1}{r+1}(1/2)^{r+1} \leq 1$, and applying Triangle Inequality gives
\[
|C| \leq \sum_{r=0}^{d-1}|a_r|
\]

Putting all of this together gives that 
\[
\sum_{r=0}^d|q_r| \leq 2\sum_{r=0}^{d-1}|a_r| \leq C_{\beta}5^d,
\]
where $C_{\beta} = 2/\beta$.

Now, the Bernstein coefficients of $Q_d$ satisfy the identity
\[
b_w
=
\sum_{r=0}^w
q_r
\frac{\binom{w}{r}}{\binom{d}{r}}.
\]
Since
\[
\frac{\binom{w}{r}}{\binom{d}{r}}
\leq
1,
\]
we have
\[
|b_w|
\leq
\sum_{r=0}^w
|q_r|
\leq
\sum_{r=0}^d
|q_r|
\leq
C_\beta5^d.
\]
The result follows.
\end{proof}

\begin{lemma}\label{lem:psi}
Let $c_{\beta},C_{\beta}>0$ be the constants from \cref{lem:defq} and \cref{lem:boundbernstein}, respectively. Then for every odd $d \geq 1$, there exists a function $\psi_d: [-1,1]\ra \mathbb{R}$ such that for every $m \neq m_0$, every context $j \in I_m$, and every $D \in [-1,1]$,
\[
|\mathbb{E}[\psi_d(\widehat{v}(i_j,t_j,m))|u_{i^j}(t^j,m)-u_{i^j}(t^j,m_0) = D]-D| \leq \exp\{-c_{\beta}d\}
\]
and
\[
\sup_{x \in [-1,1]}|\psi_d(x)| \leq C_{\beta}5^d.
\]
\end{lemma}
\begin{proof}
Let $Q_d(p) = \sum_{w=0}^d b_w\binom{d}{w}p^w(1-p)^{d-w}$ be the polynomial constructed in \cref{lem:defq}. We use its Bernstein coefficients to define $\psi_d$. Fix $m \neq m_0$ and $j \in I_m$. Let $p_j(m) = \sigma(\beta(u_{i^j}(t^j,m) -u_{i^j}(t^j,m_0)))$ denote the BT win probability of model $m$ in context $j$, and let $W_j(m)$ denote the number of wins by $m$ in the $d$ comparisons in context $j$. For each number of wins $w \in \{0,...,d\}$, let $x_w$ denote the MLE value produced by this number of wins. We define
\[
\psi_d(x) = \begin{cases}
b_w, & \text{if } x = x_w \\
0, & \text{otherwise}
\end{cases}.
\]

We verify that $\psi_d$ is well defined. To start, note that $x_1=\ldots=x_{\ell}=-1$ and $x_{d-\ell}=\ldots=x_d=1$ due to the truncation of the constrained MLE. And indeed, $b_1=\ldots=b_{\ell}$ and $b_{d-\ell}=\ldots=b_d$. 

Now, for the first part of the lemma, observe that 
\[
\begin{aligned}
\mathbb{E}\left[
\psi_d(\widehat{v}(i^j,t^j,m))
\middle|
u_{i^j}(t^j,m) -u_{i^j}(t^j,m_0) = D
\right]
&=
\sum_{w=0}^d
b_w
\mathbb{P}(W_j(m)=w\mid D)\\
&=
\sum_{w=0}^d
b_w
\binom{d}{w}
p_j(m)^w(1-p_j(m))^{d-w}\\
&=
Q_d(p_j(m)).
\end{aligned}
\]
Now, using the fact that 
\[
u_{i^j}(t^j,m) -u_{i^j}(t^j,m_0) = \frac{1}{\beta}\log \frac{p_j(m)}{1-p_j(m)}
\]
and applying \cref{lem:defq} gives
\[
\left|
\mathbb{E}\left[
\psi_d(\widehat{v}(i^j,t^j,m))
\middle|
u_{i^j}(t^j,m) -u_{i^j}(t^j,m_0)=D
\right]
-D
\right|
\leq
\exp(-c_\beta d).
\]
Since the bound in \cref{lem:defq} is uniform over
\[
p_j(m)\in[\sigma(-\beta),1-\sigma(-\beta)],
\]
the inequality holds uniformly over every $D\in[-1,1]$, every $m\neq m_0$, and every context $j\in I_m$. This proves the first part.

For the second part, \cref{lem:boundbernstein} gives
\[
\max_{w\in\{0,\ldots,d\}}
|b_w|
\leq
C_\beta5^d.
\]
Since $\psi_d$ takes only these values on attainable MLE outputs and is $0$ elsewhere,
\[
\sup_{x\in[-1,1]}
|\psi_d(x)|
\leq
C_\beta5^d.
\]
The result follows.
\end{proof}
\begin{proposition}\label{prop:mle-cal}
Fix $\eps,\delta \in (0,1)$. Then the $(\eps,\delta)$-repetition budget for $\cA_{\mathrm{MLE-CAL}}$ satisfies 
\[
d_{\cA_{\mathrm{MLE-CAL}}}(\eps,\delta) \leq O\left(\log \frac{1}{\eps} \right)
\]
\end{proposition}
\begin{proof}
By \cref{def:repbudget}, it suffices to provide a finite $n$ and model-pair sampling distribution $\cC$ under which $\cA_{\mathrm{MLE-CAL}}$ succeeds uniformly over every problem instance with minimum average utility gap $\Delta \geq \eps$. Let $m_0$ be the reference model used by $\cA_{\mathrm{MLE-CAL}}$. Define
\[
\cC := \Unif(\{\{m,m_0\}:m \in \cM \setminus m_0\}).
\]

We first choose the repetition budget $d$. Let $c_\beta>0$ be the constant from \cref{lem:psi}, and take $d$ to be the smallest odd integer satisfying
\[
\exp(-c_\beta d) \leq \frac{\eps}{8}.
\]
Then,
\[
d = O_\beta\left(\log\frac{1}{\eps}\right).
\]

By \cref{lem:psi}, there exists a function $\psi_d:[-1,1]\to\mathbb{R}$ satisfying
\[
\sup_{D\in[-1,1]}
\left|
\mathbb{E}\left[
\psi_d(\widehat{v}(i^j,t^j,m))
\middle|
u_{i^j}(t^j,m)-u_{i^j}(t^j,m_0)=D
\right]
-D
\right|
\leq
\frac{\eps}{8}.
\]
Thus, $\psi_d$ satisfies the bias constraint in the optimization defining $\phi_d$. Since $\phi_d$ minimizes the worst-case magnitude among all functions satisfying this constraint, \cref{lem:psi} also gives
\[
\sup_{x\in[-1,1]}
|\phi_d(x,\eps)|
\leq
\sup_{x\in[-1,1]}|\psi_d(x)|
\leq
C_\beta 5^d
\]
for some constant $C_\beta>0$.

Since $d$ is the smallest odd integer satisfying the bias requirement above,
\[
d \leq \frac{1}{c_\beta}\log\frac{8}{\eps}+2.
\]
It follows that there exists a constant $\kappa_\beta>0$, depending only on $\beta$, such that
\[
\left(
\sup_{x\in[-1,1]}
|\phi_d(x,\eps)|
\right)^2
\leq
C_\beta\eps^{-\kappa_\beta}.
\]

We next control the bias of the estimates calibrated by $\phi_d$. Fix $m\neq m_0$ and a context $j\in I_m$. Define the contextual utility difference
\[
D_j(m)
:=
u_{i^j}(t^j,m)-u_{i^j}(t^j,m_0).
\]
By the definition of $\phi_d$,
\[
\left|
\mathbb{E}\left[
\phi_d(\widehat{v}(i^j,t^j,m),\eps)
\middle|
i^j,t^j,j\in I_m
\right]
-
D_j(m)
\right|
\leq
\frac{\eps}{8}.
\]
Since the model pair is sampled from $\cC$ independently of the user-task context,
\[
\mathbb{E}[D_j(m)|j\in I_m]
=
\bar{u}(m)-\bar{u}(m_0).
\]
Averaging over contexts then gives
\[
\left|
\mathbb{E}\left[
\phi_d(\widehat{v}(i^j,t^j,m),\eps)
\middle|
j\in I_m
\right]
-
(\bar{u}(m)-\bar{u}(m_0))
\right|
\leq
\frac{\eps}{8}.
\]
We have controlled the bias of the calibrated MLE. It remains to control the error from averaging over finitely many sampled contexts.

Let $B_{d,\eps}
:=
\sup_{x\in[-1,1]}|\phi_d(x,\eps)|.$
Conditional on $I_m$, the random variables
\[
\left\{
\phi_d(\widehat{v}(i^j,t^j,m),\eps):j\in I_m
\right\}
\]
are independent and lie in $[-B_{d,\eps},B_{d,\eps}]$. Applying Hoeffding's Inequality to
\[
\widehat{g}_{\mathrm{MLE-CAL}}(m)
=
\frac{1}{|I_m|}
\sum_{j\in I_m}
\phi_d(\widehat{v}(i^j,t^j,m),\eps)
\]
gives
\[
\mathbb{P}\left(
\left|
\widehat{g}_{\mathrm{MLE-CAL}}(m)
-
\mathbb{E}\left[
\phi_d(\widehat{v}(i^j,t^j,m),\eps)
\middle|
j\in I_m
\right]
\right|
>
\frac{\eps}{8}
\middle|
I_m
\right)
\leq
2\exp\left(
-\frac{|I_m|\eps^2}{128B_{d,\eps}^2}
\right).
\]

Therefore, it suffices to have
\[
|I_m|
\geq
\frac{128B_{d,\eps}^2}{\eps^2}
\log\frac{4|\cM|}{\delta}
\]
for every $m\neq m_0$. Let
\[
r
:=
\left\lceil
\frac{128B_{d,\eps}^2}{\eps^2}
\log\frac{4|\cM|}{\delta}
\right\rceil.
\]
By \cref{lem:controlobs}, taking
\[
n
\geq
\max\left\{
2(|\cM|-1)r,
8(|\cM|-1)\log\frac{2(|\cM|-1)}{\delta}
\right\}
\]
guarantees that, with probability at least $1-\delta/2$,
\[
|I_m|\geq r
\]
simultaneously for all $m\neq m_0$.

Using the earlier bound on $B_{d,\eps}^2$, we have
\[
r
=
O_\beta\left(
\frac{1}{\eps^{2+\kappa_\beta}}
\log\frac{|\cM|}{\delta}
\right).
\]
Consequently, it suffices to take
\[
n
=
O_\beta\left(
\frac{|\cM|}{\eps^{2+\kappa_\beta}}
\log\frac{|\cM|}{\delta}
\right).
\]

Condition on the event that $|I_m|\geq r$ simultaneously for every $m\neq m_0$. Applying the preceding Hoeffding bound and taking a union bound over $m\neq m_0$ gives that, with conditional probability at least $1-\delta/2$,
\[
\left|
\widehat{g}_{\mathrm{MLE-CAL}}(m)
-
\mathbb{E}\left[
\phi_d(\widehat{v}(i^j,t^j,m),\eps)
\middle|
j\in I_m
\right]
\right|
\leq
\frac{\eps}{8}
\]
simultaneously for every $m\neq m_0$.

Combining the two failure probabilities with the earlier bias bound, we obtain that, with probability at least $1-\delta$,
\[
\left|
\widehat{g}_{\mathrm{MLE-CAL}}(m)
-
(\bar{u}(m)-\bar{u}(m_0))
\right|
\leq
\frac{\eps}{4}
\]
simultaneously for every $m\neq m_0$. The same bound holds trivially for $m_0$.

Take any pair of models $m,m'$. The triangle inequality then gives
\[
\left|
(\widehat{g}_{\mathrm{MLE-CAL}}(m)-\widehat{g}_{\mathrm{MLE-CAL}}(m'))
-
(\bar{u}(m)-\bar{u}(m'))
\right|
\leq
\frac{\eps}{2}.
\]
Every nonzero average utility difference has magnitude at least $\Delta\geq\eps$, so ranking according to $\widehat{g}_{\mathrm{MLE-CAL}}$ preserves every strict average-utility ordering. 
\end{proof}

\begin{proposition}\label{prop:mle-unconst}
Fix $\eps,\delta \in (0,1)$. Then the $(\eps,\delta)$-repetition budget for $\cA_{\mathrm{MLE-UNCONST}}$ satisfies 
\[
d_{\cA_{\mathrm{MLE-UNCONST}}}(\eps,\delta) \leq O\left(\log \frac{1}{\eps} \right)
\]
\end{proposition}
\begin{proof}
By \cref{def:repbudget}, it suffices to provide a finite $n$ and model-pair sampling distribution $\cC$ under which $\cA_{\mathrm{MLE-UNCONST}}$ succeeds uniformly over every problem instance with minimum average utility gap $\Delta \geq \eps$. Let $m_0$ be the reference model used by $\cA_{\mathrm{MLE-UNCONST}}$. Define 
\[
\cC := \Unif(\{\{m,m_0\}:m \in \cM \setminus m_0\})
\]
Fix $m \neq m_0$ and $j \in I_m$. After transforming the MLE estimate, we have the statistic 
\[
Z_j(m) = \frac{1}{\beta}(H_{d}-H_{d-W_j(m)}).
\]
Now, consider the following statistic:
\[
\tilde{Z}_j(m) = \frac{1}{\beta}\sum_{k=1}^{d}
\frac{
\prod_{r=1}^{k} y_r^j
-
\prod_{r=1}^{k} (1-y_r^j)
}{
k
}.
\]
Observe that we can write the conditional expectation of the products in the numerator as 
\[
\mathbb{E}\left[\prod_{r=1}^{k} y_r^j|W_j(m) = w \right] = \frac{\binom{w}{k}}{\binom{d}{k}} \qquad \text{and} \qquad \mathbb{E}\left[\prod_{r=1}^{k} (1-y_r^j)|W_j(m)=w \right] = \frac{\binom{d-w}{k}}{\binom{d}{k}}.
\]
Now, Identity 180 from \cite{proofsthatreallycount} gives that
\[
\sum_{k=1}^d \frac{\binom{w}{k}}{k\binom{d}{k}} = H_d - H_{d-w}
\]

It follows that 
\[
\mathbb{E}[\tilde{Z}_j(m)|W_j(m)] = \frac{1}{\beta}(H_{W_j(m)}-H_{d-W_j(m)}) = Z_j(m).
\]
We first compute the expectation of $Z_j(m)$, conditional on the context. Observe that by the Law of Total Expectation, $\mathbb{E}[Z_j(m)|i^j,t^j,j \in I_m] = \mathbb{E}[\tilde{Z}_j(m)|i^j,t^j,j\in I_m]$. Then, we can compute
\begin{align*}
\mathbb{E}[\tilde{Z}_j(m)|i^j,t^j,j\in I_m] =\frac{1}{\beta}
\sum_{k=1}^{d}
\frac{
p_j(m)^k-(1-p_j(m))^k
}{
k}, \\
\end{align*}
Thus, the conditional expectation of $Z_j(m)$ is the series representation of $u_{i^j}(t^j,m) - u_{i^j}(t^j,m_0)$, truncated at $d$. We then bound the bias as follows: 
\begin{align*}
|\mathbb{E}[Z_j(m)|i^j,t^j,j \in I_m]-(u_{i^j}(t^j,m) - u_{i^j}(t^j,m_0))| &\leq \frac{1}{\beta}\sum_{k>d}\frac{|p_j(m)^k-(1-p_j(m))^k|}{k} \\
\intertext{Since $p_j(m),1-p_j(m) \leq \sigma(\beta)$, this gives,}
&\leq \frac{1}{\beta}\sum_{k > d}\frac{2\sigma(\beta)^k}{k} \\
&\leq \frac{2\sigma(\beta)^{d+1}}{\beta(d+1)(1-\sigma(\beta))}.
\end{align*}
Now, if we choose
\[
d \geq \frac{\log \frac{16}{\beta(1-\sigma(\beta))\eps}}{\log \frac{1}{\sigma(\beta)}}=O\left(\log \frac{1}{\eps}\right),
\]
we get 
\[
\frac{2\sigma(\beta)^{d+1}}{\beta(d+1)(1-\sigma(\beta))} \leq \frac{\eps}{8}.
\]

Since the model pair is sampled from $\cC$, independently of the context,
\[
\mathbb{E}[u_{i^j}(t^j,m) - u_{i^j}(t^j,m_0)] = \bar{u}(m) -\bar{u}(m_0).
\]
Then, averaging over contexts gives
\[
|\mathbb{E}[Z_j(m)|j \in I_m] - (\bar{u}(m) -\bar{u}(m_0)) \leq \frac{\eps}{8}.
\]
We have controlled the bias of $Z_j(m)$. It remains to control the variance and bound $|Z_j(m)|$ before completing the proof. For the variance, observe that by Jensen's Inequality 
\[
\mathbb{E}[Z_j(m)^2|i^j,t^j,j \in I_m] \leq \mathbb{E}\left[\tilde{Z}_j(m)^2|i^j,t^j, j \in I_m \right].
\]
Then, it suffices to bound the variance of $\tilde{Z}_j(m)$ by bounding its second moment. First, we rewrite $\tilde{Z}_j(m)$ as

\[
\tilde{Z}_j(m)
=
\frac{1}{\beta}
\sum_{k=1}^{d}
\frac{
\prod_{r=1}^k y_r^j
-
\prod_{r=1}^k (1-y_r^j)
}{
k
}
\]

Then, applying Minkowski's Inequality, we get
\begin{align*}
\sqrt{
\mathbb{E}\!\left[\tilde{Z}_{j}(m)^2 \mid i^j,t^j,j \in I_m \right]
}
&\le
\frac{1}{\beta}
\sum_{k=1}^d
\frac{1}{k}
\sqrt{
\mathbb{E}\!\left[
\left(
\prod_{r=1}^k y_r^j
-
\prod_{r=1}^k(1-y_r^j)
\right)^2
\middle|i^j,t^j,j \in I_m
\right]
}\\
\intertext{Since both products above cannot simultaneously equal 1,}
&=
\frac{1}{\beta}
\sum_{k=1}^d
\frac{
\sqrt{
p_j(m)^k+(1-p_j(m))^k
}
}{
k
}\\
\intertext{Using $p_j(m),1-p_j(m) \leq \sigma(\beta)$ gives}
&\le
\frac{\sqrt{2}}{\beta}
\sum_{k=1}^d
\frac{\sigma(\beta)^{k/2}}{k}\\
&\le
\frac{\sqrt{2\rho}}{\beta}
\sum_{k=1}^{\infty}
\frac{\sigma(\beta)^{k/2}}{k}
=:V_\beta = O(1)<\infty.
\end{align*}
It follows that $\mathbb{E}[\tilde{Z}_j(m)^2|j \in I_m] \leq V_{\beta}^2$.

Finally, we can bound $|Z_j(m)|$ by 
\[
|Z_j(m)| \leq \frac{H_d}{\beta} \leq \frac{1+\log d}{\beta}.
\]
Then, by the choice of $d$, $\eps |Z_j(m)| \leq O(1)$. We can then apply Bernstein's Inequality to 
\[
\widehat{g}_{\mathrm{MLE-UNCONST}}(m) = \frac{1}{|I_m|}\sum_{j \in I_m}Z_j(m)
\]
obtaining that 
\[
\mathbb{P}\left(\widehat{g}_{|\mathrm{MLE-UNCONST}}(m)-\mathbb{E}[Z_j(m)|j \in I_m]| > \eps/8 |I_m \right) \leq 2\exp\{-c_{\beta}I_m\eps^2\}
\]
for some constant $c_{\beta} > 0$. By \cref{lem:controlobs}, taking
\[
n = O\left(\frac{|\cM|}{\eps^2}\log\frac{|\cM|}{\delta} \right)
\]
gives that with probability at least $1-\delta/2$,
\[
|I_m| \geq \frac{1}{c_\beta \eps^2}\log \frac{4|\cM|}{\delta}
\]
simultaneously for all $m \neq m_0$. 
Conditioning on this event and taking a union bound over $m \neq m_0$ gives that with probability at least $1-\delta/2$
\[|\widehat{g}_{\mathrm{MLE-UNCONST}}(m)-\mathbb{E}[{Z_j(m)}|j \in I_m]| \leq \frac{\eps}{8}\]
simultaneously for all $m \neq m_0$. Combining this with the bound on the bias of $Z_j(m)$, we obtain that 
\[
|\widehat{g}_{\mathrm{MLE-UNCONST}}(m) - (\bar{u}(m) - \bar{u}(m_0))| \leq \frac{\eps}{4}.
\]
Take any pair of models $m,m'$. Triangle inequality then implies that 
\[
|(\widehat{g}_{\mathrm{MLE-UNCONST}}(m) - \widehat{g}_{\mathrm{MLE-UNCONST}}(m')) - (\bar{u}(m) - \bar{u}(m'))| \leq \frac{\eps}{2}.
\]
It follows that ranking under $\widehat{g}_{\mathrm{MLE-UNCONST}}(m)$, which is the procedure that $\cA_{\mathrm{MLE-UNCONST}}$ uses, preserves the ranking under average utility.
\end{proof}

\mle*
\begin{proof}
Combining \cref{prop:mle-cal} and \cref{prop:mle-unconst} gives the result.
\end{proof}
\newpage
\section{Truncated Russian Roulette Algorithm} \label{appendix:randomtruncated}
\subsection{Formal Definition of $\cA_{\mathrm{TRR}}$} \label{trr-defn}
Fix an arbitrary reference model $m_0 \in \cM$. Given a dataset $S=\{(i^j,t^j,m_1^j,m_2^j,y_1^j,\ldots,y_d^j)\}_{j=1}^n$ with parameters $(n,d)$, for each $m \neq m_0$, let 
\[
I_m = \{ j: \{m_1^j,m_2^j\} = \{m,m_0\}\}
\]
be the set of contexts in which $m$ is compared to $m_0$. Without loss of generality, assume $m = m_1^j$.

Fix $\rho \in (\sigma(\beta),1)$. For each context $j$, independently draw $K_j \in \{1,\ldots,d\}$ according to
\[
\mathbb{P}(K_j \geq k) = \frac{\rho^{k-1}}{k}, \qquad k = 1,...,d.
\]

Fix $m \neq m_0$. For each $j \in I_m$, let
\[
W_j(m) = \sum_{r=1}^{K_j}y_r^j
\]
denote the number of wins by $m$ in the first $K_j$ comparisons against $m_0$ in context $j$.

Then, define the statistic \[ Z_j(m) =\frac{1}{\beta}
\sum_{k=1}^{K_j}
\frac{
\displaystyle
\frac{\binom{W_j(m)}{k}}{\binom{K_j}{k}}
-
\displaystyle
\frac{\binom{K_j-W_j(m)}{k}}{\binom{K_j}{k}}
}{
k\mathbb{P}(K_j \geq k)} = \frac{1}{\beta}\sum_{k=1}^{K_j}
\frac{
\displaystyle
\frac{\binom{W_j(m)}{k}}{\binom{K_j}{k}}
-
\displaystyle
\frac{\binom{K_j-W_j(m)}{k}}{\binom{K_j}{k}}
}{
\rho^{k-1}}. \] Now, for each model $m \neq m_0$, the estimate for the average utility relative to $m_0$ is 
\[
\widehat{g}_{\mathrm{TRR}}(m) = \frac{1}{|I_m|}\sum_{j\in I_m} Z_j(m).
\]
$\widehat{g}_{\mathrm{TRR}}(m_0) = 0$, and for completeness, if $I_m$ is empty, set $\widehat{g}_{\mathrm{TRR}}(m) = 0$. 

$\cA_{\mathrm{TRR}}$ then ranks models in decreasing order of $\widehat{g}_{\mathrm{TRR}}(m)$, breaking ties according to the fixed ordering of $\cM$.
\subsection{Proof of \cref{thm:randomtruncated}}
\randomtruncated*
\begin{proof}
By \cref{def:repbudget}, it suffices to provide a finite $n$ and model-pair sampling distribution $\cC$ under which $\cA_{\mathrm{TRR}}$ succeeds uniformly over every problem instance with minimum average utility gap $\Delta \geq \eps$. Let $m_0$ be the reference model used by $\cA_{\mathrm{TRR}}$. Define 
\[
\cC := \Unif(\{\{m,m_0\}:m \in \cM \setminus m_0\})
\]
Fix $m \neq m_0$. For each context $j \in I_m$, we can write 
\[
u_{i^j}(t^j,m) - u_{i^j}(t^j,m_0) = \frac{1}{\beta}\log \frac{p_j(m)}{1-p_j(m)}=\frac{1}{\beta} \sum_{k=1}^{\infty} \frac{p_j(m)^k - (1 - p_j(m))^k}{k},
\]
where $p_j(m) = \sigma(\beta(u_{i^j}(t^j,m) - u_{i^j}(t^j,m_0)))$ is the BT win probability for $m$ over $m_0$.
Now, consider the following statistic:
\[
\tilde{Z}_j(m) = \frac{1}{\beta}\sum_{k=1}^{K_j}
\frac{
\prod_{r=1}^{k} y_r^j
-
\prod_{r=1}^{k} (1-y_r^j)
}{
k \mathbb{P}(K_j \geq k)
}.
\]
Now, observe that 
\[
\mathbb{E}\left[\prod_{r=1}^{k} y_r^j|K_j,W_j(m) \right] = \frac{\binom{W_j(m)}{k}}{\binom{K_j}{k}} \qquad \text{and} \qquad \mathbb{E}\left[\prod_{r=1}^{k} (1-y_r^j)|K_j,W_j(m) \right] = \frac{\binom{K_j-W_j(m)}{k}}{\binom{K_j}{k}}.
\]
It follows that the statistic we use for our estimator $Z_j(m)$ can be written as
\[
Z_j(m) = \mathbb{E}[\tilde{Z}_j(m)|K_j,W_j(m)].
\]
We first compute the expectation of $Z_j(m)$, conditional on the context. Observe that by the Law of Total Expectation, $\mathbb{E}[Z_j(m)|i^j,t^j,j \in I_m] = \mathbb{E}[\tilde{Z}_j(m)|i^j,t^j,j\in I_m]$. Then, we can compute
\begin{align*}
\mathbb{E}[\tilde{Z}_j(m)|i^j,t^j,j\in I_m] &= \sum_{k=1}^d \mathbb{P}(K_j=k)\mathbb{E}[\tilde{Z}_j(m)|i^j,t^j,j\in I_m,K_j=k] \\
&=\frac{1}{\beta}
\sum_{k=1}^{d}
\mathbb{P}(K_j=k)
\sum_{\ell=1}^{k}
\frac{
p_j(m)^\ell-(1-p_j(m))^\ell
}{
\ell \mathbb{P}(K_j \geq \ell)
}
\\
&=
\frac{1}{\beta}
\sum_{\ell=1}^{d}
\frac{
p_j(m)^\ell-(1-p_j(m))^\ell
}{
\ell \mathbb{P}(K_j \geq \ell)
}
\sum_{k=\ell}^{d}
\mathbb{P}(K_j=k) \\
&=
\frac{1}{\beta}
\sum_{\ell=1}^{d}
\frac{
p_j(m)^\ell-(1-p_j(m))^\ell
}{
\ell \mathbb{P}(K_j \geq \ell)}\mathbb{P}(K_j\geq \ell) \\
&=\frac{1}{\beta}
\sum_{\ell=1}^{d}
\frac{
p_j(m)^\ell-(1-p_j(m))^\ell
}{
\ell}, \\
\end{align*}
Thus, the conditional expectation of $Z_j(m)$ is the series representation of $u_{i^j}(t^j,m) - u_{i^j}(t^j,m_0)$, truncated at $d$. We then bound the bias as follows: 
\begin{align*}
|\mathbb{E}[Z_j(m)|i^j,t^j,j \in I_m]-(u_{i^j}(t^j,m) - u_{i^j}(t^j,m_0))| &\leq \frac{1}{\beta}\sum_{k>d}\frac{|p_j(m)^k-(1-p_j(m))^k|}{k} \\
\intertext{Since $p_j(m),1-p_j(m) \leq \sigma(\beta)$, this gives,}
&\leq \frac{1}{\beta}\sum_{k > d}\frac{2\sigma(\beta)^k}{k} \\
&\leq \frac{2\sigma(\beta)^{d+1}}{\beta(d+1)(1-\sigma(\beta))}.
\end{align*}
Now, if we choose
\[
d \geq \frac{\log \frac{16}{\beta(1-\sigma(\beta))\eps}}{\log \frac{1}{\sigma(\beta)}}=O\left(\log \frac{1}{\eps}\right),
\]
we get 
\[
\frac{2\sigma(\beta)^{d+1}}{\beta(d+1)(1-\sigma(\beta))} \leq \frac{\eps}{8}.
\]

Since the model pair is sampled from $\cC$, independently of the context,
\[
\mathbb{E}[u_{i^j}(t^j,m) - u_{i^j}(t^j,m_0)] = \bar{u}(m) -\bar{u}(m_0).
\]
Then, averaging over contexts gives
\[
|\mathbb{E}[Z_j(m)|j \in I_m] - (\bar{u}(m) -\bar{u}(m_0)) \leq \frac{\eps}{8}.
\]
We have controlled the bias of $Z_j(m)$. It remains to control the variance and bound $|Z_j(m)|$ before completing the proof. For the variance, observe that by Jensen's Inequality 
\[
\mathbb{E}[Z_j(m)^2|i^j,t^j,j \in I_m] \leq \mathbb{E}\left[\tilde{Z}_j(m)^2|i^j,t^j, j \in I_m \right].
\]
Then, it suffices to bound the variance of $\tilde{Z}_j(m)$ by bounding its second moment. First, we rewrite $\tilde{Z}_j(m)$ as

\[
\tilde{Z}_j(m)
=
\frac{1}{\beta}
\sum_{k=1}^{K_j}
\frac{
\prod_{r=1}^k y_r^j
-
\prod_{r=1}^k (1-y_r^j)
}{
\rho^{k-1}
}
=
\frac{1}{\beta}
\sum_{k=1}^{d}
\mathbf 1\{K_j\ge k\}
\frac{
\prod_{r=1}^k y_r^j
-
\prod_{r=1}^k (1-y_r^j)
}{
\rho^{k-1}
}.
\]

Then, applying Minkowski's Inequality, we get
\begin{align*}
\sqrt{
\mathbb{E}\!\left[\tilde{Z}_{j}(m)^2 \mid i^j,t^j,j \in I_m \right]
}
&\le
\frac{1}{\beta}
\sum_{k=1}^d
\frac{1}{\rho^{k-1}}
\sqrt{
\mathbb{E}\!\left[
\mathbf 1\{K_j\ge k\}
\left(
\prod_{r=1}^k y_r^j
-
\prod_{r=1}^k(1-y_r^j)
\right)^2
\middle|i^j,t^j,j \in I_m
\right]
}\\
\intertext{Since both products above cannot simultaneously equal 1,}
&=
\frac{1}{\beta}
\sum_{k=1}^d
\frac{
\sqrt{
\mathbb{P}(K_j \geq k)\left(p_j(m)^k+(1-p_j(m))^k\right)
}
}{
\rho^{k-1}
}\\
\intertext{Using $p_j(m),1-p_j(m) \leq \sigma(\beta)$ and plugging in $\mathbb{P}(K_j \geq k)= \frac{\rho^{k-1}}{k}$ gives}
&\le
\frac{\sqrt{2\rho}}{\beta}
\sum_{k=1}^d
\frac{1}{\sqrt{k}}
\left(
\frac{\sigma(\beta)}{\rho}
\right)^{k/2}\\
&\le
\frac{\sqrt{2\rho}}{\beta}
\sum_{k=1}^{\infty}
\frac{1}{\sqrt{k}}
\left(
\frac{\sigma(\beta)}{\rho}
\right)^{k/2}
=:V_\beta = O(1)<\infty.
\end{align*}
In the final line, we know the series converges because $\rho > \sigma(\beta)$. It follows that $\mathbb{E}[\tilde{Z}_j(m)^2|j \in I_m] \leq V_{\beta}^2$.

Next, we obtain a bound on $|Z_j(m)|$. Note that boundedness of $|Z_j(m)|$ is immediate from the fact that $K_j$ is truncated at $d$. This would be sufficient for the remainder of the proof. However, in the interest of actually quantifying the number of sampled contexts that we need, we explicitly bound $|Z_j(m)|$.  
\[
|Z_j(m)| \leq \frac{1}{\beta}\sum_{k=1}^{K_j} \frac{1}{\rho^{k-1}} \leq \frac{1}{\beta}\sum_{k=1}^d\frac{1}{\rho^{k-1}} \leq C_{\beta}\rho^{-d}
\]
for some constant $C_{\beta}$. Then, since $d = O\left(\log \frac{1}{\eps}\right)$ and $\sigma(\beta) < \rho < 1$, there exists an exponent $a_{\beta} \in (0,1)$ such that 
\[
\rho^{-d} \leq C_{\beta}\eps^{-a_{\beta}}.
\]

We can now apply Bernstein's Inequality to 
\[
\widehat{g}_{\mathrm{TRR}}(m) = \frac{1}{|I_m|}\sum_{j \in I_m} Z_j(m),
\]
obtaining that 
\[
\mathbb{P}\left(|\widehat{g}_{\mathrm{TRR}}(m)-\mathbb{E}[{Z_j(m)}|j \in I_m]| > \frac{\eps}{8}|I_m \right) \leq 2\exp(-c_{\beta}|I_m|\eps^2)
\]
for some constant $c_\beta > 0$. By \cref{lem:controlobs}, taking
\[
n = O\left(\frac{|\cM|}{\eps^2}\log\frac{|\cM|}{\delta} \right)
\]
gives that with probability at least $1-\delta/2$,
\[
|I_m| \geq \frac{1}{c_\beta \eps^2}\log \frac{4|\cM|}{\delta}
\]
simultaneously for all $m \neq m_0$. 
Conditioning on this event and taking a union bound over $m \neq m_0$ gives that with probability at least $1-\delta/2$
\[|\widehat{g}_{\mathrm{TRR}}(m)-\mathbb{E}[{Z_j(m)}|j \in I_m]| \leq \frac{\eps}{8}\]
simultaneously for all $m \neq m_0$. Combining this with the bound on the bias of $Z_j(m)$, we obtain that 
\[
|\widehat{g}_{\mathrm{TRR}}(m) - (\bar{u}(m) - \bar{u}(m_0))| \leq \frac{\eps}{4}.
\]
Take any pair of models $m,m'$. Triangle inequality then implies that 
\[
|(\widehat{g}_{\mathrm{TRR}}(m) - \widehat{g}_{\mathrm{TRR}}(m')) - (\bar{u}(m) - \bar{u}(m'))| \leq \frac{\eps}{2}.
\]
It follows that ranking under $\widehat{g}_{\mathrm{TRR}}(m)$, which is the procedure that $\cA_{\mathrm{TRR}}$ uses, preserves the ranking under average utility. We conclude the proof by calculating the expectation of $K_j$. 
\[
\mathbb{E}[K_j] = \sum_{k=1}^d \mathbb{P}(K_j\geq k) = \sum_{k=1}^d \frac{\rho^{k-1}}{k} \leq \frac{1}{\rho}\sum_{k=1}^{\infty} \frac{\rho^k}{k} = \frac{1}{\rho}\log \frac{1}{1-\rho} = O(1).
\]
\end{proof}
\TRRsamples*
\begin{proof}
The proof of \cref{thm:randomtruncated} gives the desired bound on $n$. The result follows immediately.
\end{proof}
\newpage
\section{Lower Bound}
Throughout this section, we denote 
\[
h(p) = \frac{1}{\beta}\log \frac{p}{1-p},
\] where $\beta > 0$.
\begin{proposition}\label{prop:mus}
Fix $\beta > 0$, $d \geq 1$. Let $P_d$ denote the set of polynomials of degree at most $d$ and define 
\[
E_d = \inf_{P \in P_d}\sup_{p \in [\sigma(-\beta),1-\sigma(-\beta)]}|h(p)-P(p)|
\]
Then, there exist two distributions $\mu_+$ and $\mu_-$, both supported on $[\sigma(-\beta),1-\sigma(-\beta)]$ such that
\[
\int p^k d\mu_+(p) = \int p^k d\mu_-(p), \qquad k=0,1,...,d,
\]
and 
\[
\int h(p) d\mu_+(p) = -\int h(p) d\mu_-(p)=\gamma_d
\]
for some $\gamma_d \geq E_d$
\end{proposition}
\begin{proof}
Let $P^* \in P_d$ denote a polynomial where 
\[
||P^*-h||_{\infty} = E_d.
\]
That is, $P^*$ by definition minimizes $\ell_{\infty}$ norm with respect to $h$. Note that $E_d > 0$, since $h$ is not a degree $d$ polynomial. By Chebyshev's Alternation Theorem, this is the case if and only if there exist distinct points $p_0>p_1 >...>p_{d+1} \in [\sigma(-\beta),1-\sigma(-\beta)]$ such that 
\[
P^*(p_i)-h(p_i) = s(-1)^iE_d,
\]
where $s \in \{-1,+1\}$.
Now, for each $i$, define the coefficient 
\[
c_i = \frac{1}{\prod_{j \neq i}(p_i - p_j)}.
\]
We will show that these coefficients annihilate $P_d$. In particular, consider the Lagrange basis polynomials,
\[
L_i(x) = \prod_{j \neq i}\frac{x - p_j}{p_i-p_j}.
\]
Observe that $L_i(p_i)=1$ and $L_i(p_j)=0$ for all $j \neq i$. Now, for any $P \in P_d$, we can use Lagrange interpolation to write 
\[
P(x) = \sum_{i=0}^{d+1} P(p_i)L_i(x).
\]
Now, observe that the coefficient in front of $x^{d+1}$ in $L_i(x)$ is exactly $c_i$. Thus, the coefficient in front of $x^{d+1}$ in the right hand side above is 
\[
\sum_{i=0}^{d+1} c_i P(p_i).
\]
However, since $P$ is a polynomial of degree at most $d$, it follows that 
\[
\sum_{i=0}^{d+1} c_i P(p_i) = 0.
\]
Next, observe that since the points $p_0,...,p_{d+1}$ are ordered, the $c_i$'s alternate in sign, and the Chebyshev's Alternation Theorem statement says that $P^*(p_i)-h(p_i)$ also alternates in sign. Thus, without loss of generality, we can multiply all the $c_i$'s by $-1$ if necessary and normalize so that $\sum_i |c_i| = 1$ and 
\[
c_i(h(p_i)-P^*(p_i))=|c_i|E_d
\]
for all $i$.

Now, we proceed with actually defining the desired distributions. To start, define the signed measure 
\[
\nu = \sum_{i=0}^{d+1}c_i \delta_{p_i}
\]
Then, for any $P \in P_d$,
\[
\int P(p)d\nu(p) = 0.
\]
This gives that 
\begin{align*}
\int h(p) d\nu(p) &= \int (h(p)-P^*(p))d\nu(p) + \int P^*(p) d\nu(p) \\
&= \int (h(p)-P^*(p))d\nu(p) \\
&= \sum_{i=0}^{d+1}c_i(h(p_i)-P^*(p_i)) \\
&= E_d.
\end{align*}
Now that we have defined $\nu$, let $R_{\#}\nu = \sum_{i=0}^{d+1}c_i \delta_{1-p_i}$ denote the pushforward of $\nu$ by $R(p)=1-p$ and define a new measure $\tilde{\nu} = \nu - R_{\#}\nu$. It follows that for any $P \in P_d$
\[
\int P(p)d\tilde{\nu}(p) = \int P(p)d\nu(p) - \int P(1-p)d\nu(p) = 0,
\]
so $\tilde{\nu}$ retains the property of annihilating every polynomial in $P_d$. Next, using the fact that $h(p) = -h(1-p)$, we have that 
\[
\int h(p)d\tilde{\nu}(p) = \int h(p)d\nu(p) - \int h(1-p)d\nu(p) = 2 \int h(p) d\nu(p) = 2E_d.
\]
Now, write $\tilde{\nu} = \tilde{\nu}^+ - \tilde{\nu}^-$, where $(\tilde{\nu}^+,\tilde{\nu}^-)$ is the Jordan decomposition of $\tilde{\nu}$. Since $\tilde{\nu}$ annihilates every polynomial in $P_d$, which includes the constant polynomial, $\tilde{\nu}$ must have a mass of 0, which implies that $\tilde{\nu}^+$ and $\tilde{\nu}^-$ have identical positive mass $C$. We now have all of the components to define our distributions. Concretely, set 
\[
\mu_+ = \frac{\tilde{\nu}^+}{C} \qquad \text{and} \qquad \mu_- = \frac{\tilde{\nu}^-}{C}.
\]
Now for any $k \in \{0,...,d\}$, we know that 
\[
\int p^k d\tilde{\nu}(p) = \int p^k d\tilde{\nu}^+(p) -  \int p^k d\tilde{\nu}^-(p) = 0 \implies \int p^k d\tilde{\nu}^+(p) = \int p^k d\tilde{\nu}^-(p),
\]
since $p^k \in P_d$. Dividing by $C$ gives 
\[
\int p^k d\mu_+(p) = \int p^k d\mu_-(p),
\]
proving the first part of the proposition. For the second part, note that 
\[
R_{\#}\tilde{\nu} = R_{\#}\nu - R_{\#}R_{\#}\nu = R_{\#}\nu - \nu = -\tilde{\nu}.
\]
This implies that $R_{\#}\mu_+ = \mu_-$. Then, we have that 
\[
\int h(p)d\mu_-(p) = \int h(p)d(R_{\#}\mu_+)(p) = \int h(1-p)d\mu_+(p) = -\int h(p)d\mu_+(p) = -\gamma_d.
\]
This proves the second part of the proposition. It remains to show that $\gamma_d \geq E_d$. To see this, observe that 
\[
2C\gamma_d = \int h(p) d\tilde{\nu}(p) = 2E_d,
\]
so it suffices to show that $C \leq 1$. Indeed, since the total variation norm satisfies triangle inequality, we have that 
\[
||\tilde{\nu}||_{\TV} \leq ||\nu||_{\TV}+||R_{\#}\nu||_{\TV} \leq 2,
\]
where the final inequality follows from the definitions of the respective measures and that fact that $\sum_i |c_i| = 1$. At the same time, we have that 
\[
||\tilde{\nu}||_{\TV} = ||\tilde{\nu}^+||_{\TV} + ||\tilde{\nu}^-||_{\TV} = 2C
\]
It follows that $C \leq 1$, completing the proof.
\end{proof}

\repeatedcomparisonslowerbound*
\begin{proof}
Let $\cM = \{A,B\}$ and $\cT = \{t\}$, so $\cP_i$ is degenerate for every user $i$. Additionally, in this setting, there is only one possible model-pair sampling distribution. Namely, $\cC(\{A,B\})=1$.
Now, for a fixed user $i$, let 
\[
p_i = \sigma(\beta(u_i(t,A)-u_i(t,B))).
\]
Since $u_i(t,A)-u_i(t,B) \in [-1,1]$, we have that $p_i \in [\sigma(-\beta),1-\sigma(-\beta)]$. Now, recall that
\[
h(p_i) = \frac{1}{\beta}\log \frac{p_i}{1-p_i},
\] so we have that 
\[
u_i(t,A)-u_i(t,B) = h(p_i) \implies \bar{u}(A) - \bar{u}(B) = \mathbb{E}[h(p_i)].
\]
Although the task distribution is degenerate, users are still sampled randomly and are heterogeneous, so $p_i$ is drawn from some meta distribution $\mu$ supported on $[\sigma(-\beta),1-\sigma(-\beta)]$. For the proof, it suffices to construct two distributions, $\mu_+$ and $\mu_-$ that induce the same law over $d$ repeated comparisons, but $\mathbb{E}_{p_i \sim \mu_+}[h(p_i)]$ and $\mathbb{E}_{p_i \sim \mu_-}[h(p_i)]$ have opposite signs. 

For clarity, we again define
\[
E_d = \inf_{P \in P_d}\sup_{p_i \in [\sigma(-\beta),1-\sigma(-\beta)]}|h(p_i)-P(p_i)|,
\]
where $P_d$ denotes the set of polynomials of degree at most $d$. Indeed, \cref{prop:mus} gives us two distributions $\mu_+$ and $\mu_-$ that share the same first $d$ moments and induce opposite signs of $h(p_i)$ in expectation. In particular, 
\[
\mathbb{E}_{p \sim \mu_+}[h(p)] = \gamma_d \qquad \text{and} \qquad \mathbb{E}_{p \sim \mu_-}[h(p)] = -\gamma_d,
\]
where $\gamma_d \geq E_d > 0$
Now, we show that in a sampled context $(i^j,t)$, $d$ conditionally independent repeated Bradley-Terry comparisons must have the same joint distribution under $\mu_+$ and $\mu_-$. Let $Y^j=(y_1^j,...,y_d^j) \in \{0,1\}^d$ denote the comparison vector, where $y_r^j=1$ indicates that $A$ won in comparison $r$ over $B$. For a fixed comparison vector $y$ let $s$ denote the element-wise sum of $y$. Then,
\[
\mathbb{P}(Y^j=y|p) = p^s(1-p)^{d-s}
\]
Therefore,
\[
\mathbb{P}_{\mu_+}(Y^j=y) = \int p^s(1-p)^{d-s}d \mu_+(p).
\]
Similarly,
\[
\mathbb{P}_{\mu_-}(Y^j=y) = \int p^s(1-p)^{d-s}d \mu_-(p).
\]
In both cases, the integrand is a polynomial of degree at most $d$, meaning that it can be written entirely as a linear combination of the first $d$ moments. From the first item of \cref{prop:mus}, $\mu_+$ and $\mu_-$ have identical first $d$ moments, so it must be that 
\[
\mathbb{P}_{\mu_+}(Y^j=y) = \mathbb{P}_{\mu_-}(Y^j=y).
\]
This proves the first part of the theorem. To obtain the quantitative lower bound on $\gamma_d$, it suffices to lower bound $E_d$.
\begin{lemma}\label{lem:boundEd}
Let $r_\beta = \tanh\left(\frac{\beta}{4}\right)$. Then, 
\[
E_d = \Omega\left(\frac{r_{\beta}^d}{d} \right)
\]
\end{lemma}
\begin{proof}
Using the identity $\tanh(x)=2\sigma(2x)-1$, we can map $x \in [-1,1]$ to 
\[
p = \frac{1+\tanh(\beta/2)x}{2} \in [\sigma(-\beta),1-(\sigma(-\beta))].
\]
Then, we can write 
\[
h(p) = h\left(\frac{1+\tanh(\beta/2)x}{2}\right) = \frac{1}{\beta}\log \frac{1+\tanh(\beta/2)x}{1-\tanh(\beta/2)x}.
\]
Now, we use Chebyshev expansion to write the above function. Using the substitution $x = \cos \theta$ and identity $\tanh(\beta/2)=\frac{2\tanh(\beta/4)}{1+\tanh^2(\beta/4)}$, we have
\[
h\left(\frac{1+\tanh(\beta/2)x}{2}\right) = \frac{1}{\beta}\left[\log(1+r_{\beta}^2+2r_\beta \cos\theta)-\log(1+r_{\beta}^2-2r_\beta \cos\theta) \right].
\]
Then, we use the series expansion $\log(1+r^2-2r\cos\theta) = -2\sum_{k=1}^{\infty} \frac{r^k}{k}\cos(k\theta)$ to obtain the desired Chebyshev Expansion,
\[
h(p) = \frac{4}{\beta}\sum_{k \geq 1, k\text{ odd}}\frac{r_{\beta}^k}{k}\cos(k\theta).
\]
Take any $P \in P_d$. Now, let $k_d$ denote the smallest odd integer strictly greater than $d$. By definition, in the Chebyshev expansion of $P$, the coefficient in front of $\cos(k_d \theta)$ is 0. On the other hand, the coefficient in front of  $\cos(k_d \theta)$ in the Chebyshev expansion of $h$ above is 
\[
\frac{4r_{\beta}^{k_d}}{\beta k_d}.
\]
Putting this together and using the fact that for any continuous function $f$, the $k$'th Chebyshev coefficient $a_k(f)$ is such that $|a_k(f)| \leq 2||f||_{\infty}$, we get 
\[
||h-P||_{\infty} \geq \frac{1}{2}\frac{4r_{\beta}^{k_d}}{\beta k_d} \geq \frac{2r_{\beta}^{k_d}}{\beta k_d}
\]
Taking the infimum over $P \in P_d$ then gives that 
\[
E_d \geq \frac{2r_{\beta}^{k_d}}{\beta k_d},
\]
and the result follows by the fact that $d < k_d \leq d+2$.
\end{proof}
We conclude the proof of the theorem by constructing the two random instances, $\cI_+$ and $\cI_-$. Fix any finite number of sampled contexts $n$, and let $\cD$ be the uniform distribution over $N$ users, where $N$ will be chosen later. 

Again, consider the distributions $\mu_+$ and $\mu_-$ provided by \cref{prop:mus}. Under instance $\cI_+$, draw $p_1,...,p_N \overset{\mathrm{i.i.d.}}{\sim} \mu_+$, and under $\cI_-$, draw $p_1,...,p_N \overset{\mathrm{i.i.d.}}{\sim} \mu_-$. These values are drawn once when constructing the instance and then held fixed when generating the dataset $S$. Probabilities under $\cI_+$ and $\cI_-$ average over this initial assignment of $p_1,...,p_N$, the sampling of $S$, and any randomness involved with a ranking algorithm. 

In both instances, for each user $i$, define 
\[
u_i(t,A) = \frac{1+h(p_i)}{2} \qquad \text{and} \qquad u_i(t,B) = \frac{1-h(p_i)}{2}.
\]
Conditional on all $n$ sampled users being distinct, the corresponding $p_i$'s are drawn independently from $\mu_+$ or $\mu_-$, which share the same first $d$ moments as we have shown. Then, $d$ repeated comparisons share the same distribution under the two instances. Subsequently, since the users are sampled from the same distribution $\cD$, conditional on all the sampled users being distinct, the dataset $S$ generated with parameters $(n,d)$ has the same distribution under the two instances. To see that the probability of sampling the same user twice can be arbitrarily small, observe that the probability is upper bounded by 
\[
\frac{\binom{n}{2}}{N},
\]
and we can choose $N$ arbitrarily large.

The final part of the proof involves showing that the two instances induce opposite average-utility rankings, and have the required average utility gap. To start, we can write
\[
\bar{u}(A) - \bar{u}(B) = \frac{1}{N}\sum_{i=1}^N h(p_i).
\]
By \cref{prop:mus}, under instance $\cI_+$, $\mathbb{E}[h(p_i)] = \gamma_d$ and under $\cI_-$, $\mathbb{E}[h(p_i)] = -\gamma_d$

Now, since $h(p_i) \in [-1,1]$, whenever $\gamma_d > \eps$, applying Hoeffding's Inequality gives that 
\[
\mathbb{P}_{\cI_+}(\bar{u}(A)-\bar{u}(B) < \eps) \leq \exp\left(-\frac{N(\gamma_d-\eps)^2}{2} \right)
\]
and analogously,
\[
\mathbb{P}_{\cI_-}(\bar{u}(A)-\bar{u}(B) > -\eps) \leq \exp\left(-\frac{N(\gamma_d-\eps)^2}{2} \right).
\]
Taking $N$ large enough makes both of these probabilities arbitrarily small.

Now, consider a ranking algorithm $\cA$, using a dataset $S$ with parameters $(n,d)$, that errs with probability at most $\delta < 1/2$ on every instance with average utility gap at least $\eps$. Since the induced distributions of $S$ under $\cI_+$ and $\cI_-$ have total variation distance at most $\binom{n}{2}/N$, it must be that  
\[
\frac{1}{2}\left(\mathbb{P}_{I_+}(B \succ_{\cA(S)} A) + \mathbb{P}_{I_-}(A \succ_{\cA(S)} B) \right) \geq \frac{1}{2}\left(1 - \frac{\binom{n}{2}}{N}\right).
\]
We can also upper bound this average error by 
\[
\delta + \exp\left(-\frac{N(\gamma_d-\eps)^2}{2} \right),
\]
where the first term comes from the algorithm's guarantee on instances with average utility gap at least $\epsilon$, and the second term comes from assuming that the algorithm always errs on instances with average utility gap less than $\epsilon$. As $N$ increases, the upper bound on the average error then becomes arbitrarily close to $\delta < 1/2$ and the lower bound becomes arbitrarily close to $1/2$. Thus, there exists some $N$ for which the bounds cross, contradicting that $\delta < 1/2$. Therefore, if $\gamma_d > \eps$, no finite $n$ can provide the guarantee required in \cref{def:repbudget}. Thus, we must have that 
\[
\gamma_d \leq \eps
\]

From the proof of \cref{lem:boundEd}, this means that 
\[
\frac{2r_{\beta}^{k_d}}{\beta k_d} \leq \eps.
\]
Taking logs and rearranging gives 
\[
\log \frac{2}{\beta\eps} \leq k_d \log \frac{1}{r_\beta} + \log k_d \leq k_d\left(\log \frac{1}{r_\beta} + 1\right) \implies k_d > \frac{\log \frac{2}{\beta\eps}}{\log \frac{1}{r_\beta} + 1} = \Omega\left(\log \frac{1}{\eps} \right).
\]
Since $d < k_d \leq d+2$, we chose $\cA$ arbitrarily, and $\cC$ is the only possible model-pair sampling distribution in this setting, it follows that 
\[
d^*(\eps,\delta) = \Omega\left(\log \frac{1}{\eps} \right).
\]
\end{proof}
\newpage
\section{Untruncated Russian Roulette (URR) Algorithm}\label{appendix:randomizedestimator}
Here, we describe in detail an untruncated version of the Russian Roulette estimator presented in the main paper. Unlike $\cA_{\mathrm{TRR}}$, the number of repeated comparisons used within a context is now allowed to be unbounded. Thus, this estimator does not admit a deterministic repetition-budget guarantee, but it continues to use only a constant number of comparisons per context in expectation.

Fix an arbitrary reference model $m_0\in\cM$. Given a dataset $S=\{(i^j,t^j,m_1^j,m_2^j,y_1^j,\ldots,y_d^j)\}_{j=1}^n$ with parameters $(n,d)$, for each $m\neq m_0$, let
\[
I_m
=
\{j\in[n]:\{m_1^j,m_2^j\}=\{m,m_0\}\}
\]
denote the set of sampled contexts in which $m$ is compared to $m_0$. For notational convenience, for each $j\in I_m$, relabel the binary comparison outcomes if necessary so that $y_r^j=1$ indicates that $m$ wins the $r$th comparison against $m_0$.

Fix $\rho\in(\sigma(\beta),1)$. Independently for each sampled context $j$, draw
\[
K_j\in\{1,2,\ldots\}
\]
according to the geometric distribution
\[
\Pr(K_j=k)
=
(1-\rho)\rho^{k-1}.
\]
Conditional on the context and $K_j$, obtain $K_j$ independent Bradley--Terry comparisons of the sampled model pair. For $j\in I_m$, define
\[
Z_j(m)
=
\frac{1}{\beta}
\frac{
\prod_{r=1}^{K_j}y_r^j
-
\prod_{r=1}^{K_j}(1-y_r^j)
}{
K_j(1-\rho)\rho^{K_j-1}
}.
\]

We aggregate these statistics using a median-of-means estimator. Fix an odd integer $B\geq 1$. For each $m\neq m_0$, if $|I_m|\geq B$, partition $I_m$ into $B$ disjoint blocks
\[
I_{m,1},\ldots,I_{m,B}
\]
of common size
\[
b_m
=
\left\lfloor\frac{|I_m|}{B}\right\rfloor,
\]
discarding at most $B-1$ contexts. Define the corresponding block means
\[
\overline Z_s(m)
=
\frac{1}{b_m}
\sum_{j\in I_{m,s}}Z_j(m),
\qquad
s\in[B],
\]
and let
\[
\widehat g_{\mathrm{URR}}(m)
=
\operatorname{median}_{s\in[B]}
\overline Z_s(m).
\]
Set $\widehat g_{\mathrm{URR}}(m_0)=0$. For completeness, if $|I_m|<B$, set $\widehat g_{\mathrm{URR}}(m)=0$. The corresponding ranking algorithm, $\cA_{\mathrm{URR}}$ ranks models in decreasing order of $\widehat g_{\mathrm{URR}}(m)$, breaking ties according to the fixed ordering of $\cM$.
\begin{theorem}
Fix $\eps,\delta \in (0,1)$. There exists a model-pair sampling distribution $\cC$ such that $\cA_{\mathrm{URR}}$ recovers the average-quality ranking with probability at least $1-\delta$, uniformly over every problem instance with $\Delta \geq \eps$, using 
\[
n = O\left(\frac{1}{\eps^2}\log \frac{|\cM|}{\delta} \right)
\]
sampled contexts. Moreover, letting $K_j$ denote the number of repeated comparisons used in context $j$, \[\mathbb{E}[K_j] = O(1).\] Consequently, the expected total number of comparisons is 
\[
O\left(\frac{|\cM|}{\eps^2}\log \frac{|\cM|}{\Delta} \right).
\]
\end{theorem}
\begin{proof}
Define 
\[
\cC = \Unif(\{\{m,m_0\}: m \in \cM \setminus m_0\}).
\]
Fix $m \neq m_0$ and $j \in I_m$.
Let
$
p_j(m)
=
\sigma\left(
\beta\left(
u_{i^j}(t^j,m)-u_{i^j}(t^j,m_0)
\right)
\right)
$
denote the Bradley--Terry win probability for $m$ against $m_0$ in this context. Then,
\[
u_{i^j}(t^j,m)-u_{i^j}(t^j,m_0)
=
\frac{1}{\beta}\log\frac{p_j(m)}{1-p_j(m)}
=
\frac{1}{\beta}
\sum_{k=1}^{\infty}
\frac{p_j(m)^k-(1-p_j(m))^k}{k}.
\]

We first compute the expectation of $Z_j(m)$ conditional on the context.
By the definition of $K_j$,
\[
\begin{aligned}
\mathbb E[Z_j(m)\mid i^j,t^j,j\in I_m]
&=
\sum_{k=1}^{\infty}
(1-\rho)\rho^{k-1}
\frac{
p_j(m)^k-(1-p_j(m))^k
}{
\beta k(1-\rho)\rho^{k-1}
}\\
&=
\frac{1}{\beta}
\sum_{k=1}^{\infty}
\frac{
p_j(m)^k-(1-p_j(m))^k
}{k}\\
&=
u_{i^j}(t^j,m)-u_{i^j}(t^j,m_0).
\end{aligned}
\]
Thus, $Z_j(m)$ is an unbiased estimator of the context-level utility difference. Since the sampled model pair is independent of the user-task context,
\[
\mathbb E[Z_j(m)\mid j\in I_m]
=
\bar{u}(m)-\bar{u}(m_0).
\]
Thus, $Z_j(m)$ is an unbiased estimator of the average utility difference
between $m$ and $m_0$.

Next, we bound its variance by bounding its second moment. Conditional on
$K_j=k$, the events that $m$ wins all $k$ comparisons and that $m_0$ wins
all $k$ comparisons are disjoint. Therefore,
\[
\mathbb E[Z_j(m)^2\mid i^j,t^j,j\in I_m]
=
\frac{1}{\beta^2}
\sum_{k=1}^{\infty}
\frac{
p_j(m)^k+(1-p_j(m))^k
}{
k^2(1-\rho)\rho^{k-1}
}.
\]
Since the utility difference lies in $[-1,1]$,
$
p_j(m),\,1-p_j(m)\leq\sigma(\beta).
$
Hence,
\[
\begin{aligned}
\mathbb E[Z_j(m)^2\mid i^j,t^j,j\in I_m]
&\leq
\frac{2\rho}{\beta^2(1-\rho)}
\sum_{k=1}^{\infty}
\frac{1}{k^2}
\left(
\frac{\sigma(\beta)}{\rho}
\right)^k\\
&\leq
\frac{2\rho}{\beta^2(1-\rho)}
\sum_{k=1}^{\infty}
\left(
\frac{\sigma(\beta)}{\rho}
\right)^k\\
&=:V_\beta<\infty,
\end{aligned}
\]
where the series converges because $\rho>\sigma(\beta)$. Averaging over
the context gives
\[
\operatorname{Var}(Z_j(m)\mid j\in I_m)\leq V_\beta.
\]

We now control the median-of-means estimator. Choose $B$ to be the
smallest odd integer satisfying
\[
B\geq 8\log\frac{2|\cM|}{\delta},
\]
and let
\[
b
=
\left\lceil
\frac{64V_\beta}{\eps^2}
\right\rceil.
\]
By \cref{lem:controlobs}, if
\[
n
\geq
\max\left\{
2(|\cM|-1)Bb,\,
8(|\cM|-1)
\log\frac{2(|\cM|-1)}{\delta}
\right\},
\]
then with probability at least $1-\delta/2$,
\[
|I_m|\geq Bb
\]
simultaneously for every $m\neq m_0$.

Condition on this event. Since
\[
b_m
=
\left\lfloor
\frac{|I_m|}{B}
\right\rfloor
\geq b,
\]
each block mean satisfies
\[
\mathbb E[\overline Z_s(m)\mid I_m]
=
\bar{u}(m)-\bar{u}(m_0)
\]
and
\[
\operatorname{Var}(\overline Z_s(m)\mid I_m)
\leq
\frac{V_\beta}{b_m}
\leq
\frac{V_\beta}{b}.
\]
Thus, Chebyshev's Inequality gives
\[
\Pr\left(
\left|
\overline Z_s(m)-(\bar{u}(m)-\bar{u}(m_0))
\right|
>
\frac{\eps}{4}
\;\middle|\;
I_m
\right)
\leq
\frac{16V_\beta}{b\eps^2}
\leq
\frac14.
\]

For the median of the $B$ block means to differ from $u(m)-u(m_0)$ by
more than $\eps/4$, at least half of the block means must differ by more
than $\eps/4$. Since the blocks are disjoint, Hoeffding's Inequality gives
\[
\Pr\left(
\left|
\widehat g_{\mathrm{URR}}(m)
-
(u(m)-u(m_0))
\right|
>
\frac{\eps}{4}
\;\middle|\;
I_m
\right)
\leq
\exp(-B/8)
\leq
\frac{\delta}{2|\cM|}.
\]
Taking a union bound over $m\neq m_0$ gives that, with conditional
probability at least $1-\delta/2$,
\[
\left|
\widehat g_{\mathrm{URR}}(m)
-
(\bar{u}(m)-\bar{u}(m_0))
\right|
\leq
\frac{\eps}{4}
\]
simultaneously for every $m\neq m_0$. Combining this with the earlier
event controlling the sizes of the sets $I_m$, the same statement holds
with probability at least $1-\delta$.

Take any pair of models $m,m'$. The triangle inequality gives
\[
\left|
\left(
\widehat g_{\mathrm{URR}}(m)
-
\widehat g_{\mathrm{URR}}(m')
\right)
-
\left(
\bar{u}(m)-\bar{u}(m')
\right)
\right|
\leq
\frac{\eps}{2}.
\]
Every nonzero average utility difference has magnitude at least
$\Delta\geq\eps$, so ranking according to $\widehat g_{\mathrm{URR}}$
preserves every strict average-utility ordering.

Finally,
\[
B
=
O\left(
\log\frac{|\cM|}{\delta}
\right)
\qquad\text{and}\qquad
b
=
O\left(
\frac{1}{\eps^2}
\right),
\]
so the choice of $n$ above gives
\[
n
=
O\left(
\frac{|\cM|}{\eps^2}
\log\frac{|\cM|}{\delta}
\right).
\]

It remains to calculate the expected number of repeated comparisons.
Since $K_j$ is geometric with success probability $1-\rho$,
\[
\mathbb E[K_j]
=
\frac{1}{1-\rho}
=
O(1).
\]
Since only one model pair is sampled in each context, the total number of
pairwise comparisons is $\sum_{j=1}^n K_j$. Therefore,
\[
\mathbb E\left[
\sum_{j=1}^n K_j
\right]
=
\frac{n}{1-\rho}
=
O\left(
\frac{|\cM|}{\eps^2}
\log\frac{|\cM|}{\delta}
\right),
\]
which concludes the proof.

\end{proof}

The key difference between $\widehat{g}_{\mathrm{URR}}$ and the truncated Russian-Roulette estimator $\widehat{g}_{\mathrm{TRR}}$ is that the number of repeated comparisons $K_{j}$ is now allowed to be unbounded, so we can no longer obtain a deterministic upper bound on the number of comparisons. The benefit, however, is that we can define this estimator agnostic of $\Delta$, which could be useful in cases where the minimum average utility gap between models in the class is unclear. In terms of the proof, the statistic is now an unbiased estimator of the average utility difference, but it can be unbounded. Thus, in order to obtain the desired logarithmic dependence on $|\cM|/\delta$, we use a median of means estimator and Chebyshev's Inequality for the bounds. 
\newpage
\section{Experiments}\label{appendix:experiments}
In this section, we describe in detail our experimental methodology as well as additional results not included in the main paper. 
\subsection{Observation Model}
We don't distinguish between users and tasks, simply modeling them as contexts $c$. A population consists of model utilities, where model $m$ has utility $u_c(m)$ for context $c$, and context weights $w_c$. The average utility of model $m$ is then 
\[
\bar{u}(m) = \sum_{c}w_c u_c(m),
\]
and we seek to recover the ranking of models under average utility. Let $\Delta$ denote the minimum nonzero average utility gap between any two models in the class. 

For all of the experiments, we fix a reference model $m_0$. For each of the $n$ samples, we independently draw a context $c$ according to $w_c$, uniformly select one of the models $m \neq m_0$, and compare that model with $m_0$. Conditional on the context $c$ and selected model $m$, comparison outcomes are independent Bernoulli draws with success probability 
\[
p_c(m) = \sigma(\beta(u_c(m) - u_c(m_0))).
\]
We use $\beta=1$ throughout. Repeated comparisons within a sample use the same context and model pair. None of our algorithms observe the latent utilities, they only observe the dataset $S$ described in \cref{sec:setup}.

The three MLE-based algorithms deterministically use $d$ comparisons per sample, whereas the truncated Russian Roulette algorithm uses a random number of comparisons. Thus, for this algorithm, we report both the maximum number of repeated comparisons it uses as well as the expected number.

\subsection{Synthetic Experiment Construction}
We consider 5 models, and the reference model is arbitrarily chosen to be model 4. For a scalar $D$, define the utility profile
\[
 u(D)=\left(0.8,\;0.65,\;\frac{1+D}{2},\;\frac{1-D}{2},\;0.2\right).
\]
We use three context types, corresponding to $D\in\{\Delta,1,-0.49\}$, with probabilities
\[
 \Pr(D=\Delta)=1-h,\qquad
 \Pr(D=1)=h\alpha,\qquad
 \Pr(D=-0.49)=h(1-\alpha),\qquad
 \alpha=\frac{\Delta+0.49}{1.49}
 \label{eq:app-synthetic-mixture},
\]
where $h \in [0,1]$. Since $\mathbb{E}[D]=\Delta$, every value of $h$ gives the same average utility vector
\[
 \bar u=\left(0.8,\;0.65,\;0.5+\Delta/2,\;0.5-\Delta/2,\;0.2\right).
\]
For all tested gaps, the target ranking is $1\succ2\succ3\succ4\succ5$, with minimum gap $\Delta$ between models $3$ and $4$. At $h=0$ the population is homogeneous. Increasing $h$ places more probability on two opposing preference profiles without changing the means. In particular,
\[
 \operatorname{Var}(D)
 =h\left[\alpha(1-\Delta)^2+(1-\alpha)(-0.49-\Delta)^2\right],
\]
so $h$ adjusts the variance linearly.

For the first synthetic experiment, we consider a grid of $h$ values $h \in \{0,0.1,0.2,\ldots,1\}$ and fix $\Delta = 0.015$. For the second experiment, we fix $h = 1$ and sweep over a grid of $\Delta$ values $\Delta \in \Delta\in\{0.06,0.04,0.03,0.02,0.015,0.01,0.0075,0.005,0.003,0.002\}$.

\subsection{Semi-Synthetic Experiment Construction} We construct context types by clustering prompts from the Arena dataset into 30 groups. We represent each prompt by embedding its last user message with paraphrase-multilingual-MiniLM-L12-v2. Each cluster represents a context type, and its empirical weight $\pi_c$ is the fraction of retained prompts assigned to it. Within each cluster, we fit a regularized Bradley–Terry model to the human preference votes with a decisive winner. The resulting scores describe how each model’s relative performance varies across context types.

We fit Bradley–Terry scores for 52 LLMs using the human preference votes within each prompt cluster. We then select five models with sufficient observations across all clusters, requiring each model to appear in at least 100 comparisons overall and at least five comparisons in every cluster. The selected models are: \texttt{claude-opus-4-20250514} (1), \texttt{gemini-2.5-pro} (2),
\texttt{gemini-2.5-flash} (3), \texttt{qwen3-235b-a22b-no-thinking} (4), and
\texttt{mistral-medium-2505} (5). We use \texttt{mistral-medium-2505} as the reference model.

\paragraph{Utility Functions.} Let $S_c(m)$ be the saved fitted scores for the selected five models. Remove each cluster's mean score, then center each model under the empirical context distribution:
\begin{align*}
 T_c(m)&=S_c(m)-\frac{1}{5}\sum_{j=1}^{5}S_c(j),\\
 R_c(m)&=\frac{T_c(m)-\sum_{c'}\pi_{c'}T_{c'}(m)}
 {\max_{c',j}\left|T_{c'}(j)-\sum_{c''}\pi_{c''}T_{c''}(j)\right|}.
\end{align*}
Then, we have the fixed average utility ranking \[\bar{u}=(0.510,0.505,0.500,0.495,0.490)\] and context-level utilities are
\[
 u_c(m)=\bar{u}(m)+0.4554\,R_{c}(m). \]
Thus, $u_c(m)$ is designed to lie in $[0,1]$.

\paragraph{Hard Context Distribution.} We construct a hard reweighting of the context distribution. In particular, let $a_c$ be $\cA_{\mathrm{MLE-AVG}}$'s expected score difference between models 4 and 5 after one comparison. Since model 5 is the reference and $\beta = 1$, we have that 
\[
a_c = 2\sigma(u_c(4)-u_c(5)) - 1
\]
Then, we solve the linear program 
\[
\begin{aligned}
 \underset{w}{\operatorname{minimize}}\quad &\sum_c w_ca_c\\
 \text{subject to}\quad
 &\sum_c w_c=1\\
 & \sum_c w_cR_c(m)=0 \qquad \text{for every model $m$} \\
 &\pi_c/5\leq w_c\leq5\pi_c\quad\text{for every }c.
 \end{aligned}
\]
The constraints guarantee that the reweighting is a valid distribution, the average utilities of the models are unchanged, and the weights can change by at most a factor of 5 relative to the empirical distribution. We call the distribution emerging from this LP $w^{\text{hard}}$. 

Now, for the experiment, we introduce a parameter $\lambda \in [0,1]$ that interpolates between $\pi$ and $w^{\text{hard}}$. In particular, for fixed $\lambda$, the resulting context distribution is 
\[
(1-\lambda)\pi + \lambda w^{\text{hard}}.
\]
In our experiment, we consider $\lambda \in \{0,0.1,0.2,\ldots,1\}$ and keep the average utilities fixed with $\Delta = 0.005$.

\subsection{Algorithm Implementations}
The implementations of $\cA_{\mathrm{MLE-AVG}}$, $\cA_{\mathrm{MLE-UNCONST}}$, and $\cA_{\mathrm{TRR}}$ are straightforward given the formal definitions of the algorithms in \cref{mle-avg-defn}, \cref{mle-unconst-defn}, and \cref{trr-defn}, respectively. For $\cA_{\mathrm{TRR}}$, we set $\rho = (1+ \sigma(\beta))/2$. 

Now, we focus on describing how we find the calibration function for $\cA_{\mathrm{MLE-CAL}}$.

\paragraph{Finding the calibration function.}
For a fixed repetition depth $d$, the win count can take only the values $0,1,\ldots,d$, so the constrained MLE estimate also has finitely many possible values, as we have discussed throughout the paper. We enumerate these estimates and a calibrated value to each distinct estimate, adhering to the constraints in \cref{lem:psi}. Win counts that produce the same constrained MLE estimate must receive the same calibrated value. Since we are only working with finitely many outputs, constructing the calibration function simply amounts to defining a lookup table for each setting of $d,\epsilon,\beta$ mapping the at most $d+1$ possible MLE estimates to their calibrated estimates.

We choose the table entries by linear programming. For any proposed true utility difference $D\in[-1,1]$, the win count follows a binomial distribution with $d$ trials and success probability $\sigma(\beta D)$. We can therefore compute the expected calibrated estimate by summing the table outputs over all win counts, weighted by their binomial probabilities. We require this expectation to lie within $\epsilon/8$ of $D$. This gives two linear inequalities in the unknown table entries for each $D$. Among tables satisfying these constraints, we minimize the largest absolute table entry. Throughout the experiments, we use $\epsilon=\Delta/2$, so the target bias is at most $\Delta/16$.

\paragraph{Numerical solution and verification.}
We initially impose the bias constraints on $129$ equally spaced values of $D$ in $[0,1]$; the sign symmetry covers negative values. We solve the resulting linear program with HiGHS, then search between grid points for violations of the bias requirement. When we find a violation, we add that value of $D$ to the constraints and solve again. Before finalizing a lookup table, we check the bias over the entire interval $[-1,1]$ using interval subdivision and derivative bounds, with numerical tolerance $2\times10^{-8}$. If no table passes these checks, that depth is viewed as an insufficient repetition depth for the algorithm.

The table depends on $d$, $\beta$, and $\epsilon$, but not on the context distribution or utility profiles. We reuse a table whenever these parameters are unchanged, and recompute it as $\Delta$ changes in the second synthetic experiment. During aggregation, we replace each context's constrained MLE estimate by its table entry before averaging across contexts.
\subsection{Estimating the Repetition Budget in the Experiments}
\label{app:cap-selection}

For each nonreference model, the deterministic estimator's win count distribution is computed by summing over the finite context support:
\[
 \Pr(W=j\mid m)=\sum_c w_c\binom{d}{j}p_c(m)^{j}(1-p_c(m))^{d-j}.
\]
For TRR we similarly enumerate $(K,W)$ using its stopping probabilities and the corresponding binomial mixture. Means and variances are obtained from these finite distributions, up to floating-point numerical error; no sampled contexts are used to estimate the population ordering.

Let $s_m^{A,d}$ be method $A$'s expected score at depth $d$, with $s_{m_0}^{A,d}=0$, and let $m_1,\ldots,m_5$ denote the decreasing true-utility order. We define
\[
 \gamma_{A,d}=\min_{1\leq j<5}
       \left(s_{m_j}^{A,d}-s_{m_{j+1}}^{A,d}\right)
\]
and report the smallest $d$ with $\gamma_{A,d}>10^{-12}$. A positive margin yields consistent full-ranking recovery as $n$ increases. It does not by itself specify a useful finite context count.

We use the grid
\[
 \begin{split}
 d\in \{&1,2,4,8,12,16,24,32,48,64,96,128,\\
              &192,256,384,512,768,1024,1536,2048\}.
 \end{split}
\]
For the experiment varying $\Delta$, we append $3072,4096,6144,8192,12288,16384,24576,32768$. Repetition depths are checked in increasing order, stopping separately for each method at its first success. 

We report the first successful $d$ from the above grid, so this is an estimated value, not necessarily the optimal repetition budget.

\subsection{Results}
We provide the underlying tables for the figures provided in the main paper as well as an additional fixed-$d$ experiment.
\begin{table}[!h]
\centering\small
\caption{The reported values are the first $d$ for which the algorithm succeeded on some finite $n$. This may not be the optimal $d$ to minimize total comparisons. The MLE-based methods use exactly their reported depth per context, while TRR uses a randomized number. Thus, we report both its maximum and expected number of repeated comparisons.}
\label{tab:app-population-caps}
\begin{tabular}{lrrrrr}
\toprule
Setting & AVG & CAL & UNCONST & TRR cap & TRR $\mathbb{E}[K]$ \\
\midrule
Synthetic: $h=0,0.1,\ldots,0.7$ & 1 & 12 & 1 & 1 & 1.0000 \\
Synthetic: $h=0.8,0.9$ & 384 & 12 & 4 & 4 & 1.8446 \\
Synthetic: $h=1$ & 512 & 12 & 4 & 4 & 1.8446 \\
Arena: $\lambda=0,0.1,\ldots,0.8$ & 1 & 16 & 1 & 1 & 1.0000 \\
Arena: $\lambda=0.9,1$ & 96 & 16 & 4 & 4 & 1.8446 \\
\bottomrule
\end{tabular}
\end{table}

\paragraph{Varying Heterogeneity and the Context Distribution.} \cref{tab:app-population-caps} provides the results from the first synthetic experiment and the semi-synthetic experiment, corresponding to \cref{fig:repetition-budget}. We see that $\cA_{\mathrm{MLE-AVG}}$ uses as many as 512 repeated comparisons in the synthetic regime and 96 repeated comparisons in the Arena regime. Meanwhile, the other three algorithms remain relatively stable with respect to varying $h$ or $\lambda$.
\paragraph{Varying the Minimum Average Utility Gap.} \cref{tab:app-gap-caps} provides the results for the synthetic experiment in which we hold $h=1$ fixed and vary $\Delta$. 

\begin{table}[h]
\centering\small
\caption{Complete synthetic gap sweep at $h=1$. The reported values are the first $d$ for which the algorithm succeeded on some finite $n$. This may not be the optimal $d$ to minimize total comparisons. The MLE-based methods use exactly their reported depth per context, while TRR uses a randomized number. Thus, we report both its maximum and expected number of repeated comparisons.}
\label{tab:app-gap-caps}
\begin{tabular}{rrrrrr}
\toprule
$\Delta$ & AVG & CAL & UNCONST & TRR cap & TRR $\mathbb{E}[K]$ \\
\midrule
$0.06$ & 1 & 12 & 1 & 1 & 1.0000 \\
$0.04$ & 1 & 12 & 1 & 1 & 1.0000 \\
$0.03$ & 1 & 12 & 1 & 1 & 1.0000 \\
$0.02$ & 1 & 12 & 1 & 1 & 1.0000 \\
$0.015$ & 512 & 12 & 4 & 4 & 1.8446 \\
$0.01$ & 1,024 & 12 & 8 & 8 & 2.1433 \\
$0.0075$ & 2,048 & 16 & 8 & 8 & 2.1433 \\
$0.005$ & 4,096 & 16 & 8 & 8 & 2.1433 \\
$0.003$ & 12,288 & 16 & 12 & 12 & 2.2441 \\
$0.002$ & 24,576 & 16 & 12 & 12 & 2.2441 \\
\bottomrule
\end{tabular}
\end{table}

\paragraph{Fixed $d$ Experiment.} To examine whether increasing the number of contexts can compensate for a fixed repetition budget, we set $d = 32$ for all four algorithms. We consider $(h,\Delta) = (1,0.015)$ for the synthetic regime and $(\lambda,\Delta) = (1,0.005)$ for the Arena regime. The MLE-based algorithms use all 32 comparisons, while TRR is randomized with truncation level $d=32$.

\begin{table}[!h]
\centering\small
\caption{Fixed $d=32$ experiment. For the synthetic regime $(h,\Delta)=(1,0.015)$ and Arena regime $(\lambda,\Delta)=(1,0.005)$.}
\label{tab:g6-fixed-cap}
\begin{tabular}{llrrrr}
\toprule
Population & Method & Min. score gap & $n=10^5$ & $n=10^6$ & $n=10^7$ \\
\midrule
Synthetic & MLE-AVG & $-0.0324717$ & 0.00000 & 0.00000 & 0.00000 \\
Synthetic & MLE-CAL & $0.0143426$ & 0.98100 & 1.00000 & 1.00000 \\
Synthetic & MLE-UNCONST & $0.0149989$ & 0.99850 & 1.00000 & 1.00000 \\
Synthetic & TRR & $0.0149989$ & 0.90925 & 0.99975 & 1.00000 \\
\addlinespace
Arena & MLE-AVG & $-0.0064686$ & 0.02775 & 0.00000 & 0.00000 \\
Arena & MLE-CAL & $0.0046866$ & 0.14200 & 0.82775 & 1.00000 \\
Arena & MLE-UNCONST & $0.0050000$ & 0.60275 & 0.99950 & 1.00000 \\
Arena & TRR & $0.0050000$ & 0.08775 & 0.63750 & 0.99975 \\
\bottomrule
\end{tabular}
\end{table}

Keeping the repetition budget fixed, we vary contexts over $n \in \{10^5,10^6,10^7\}$. For each population, algorithm, and number of contexts, we run 4000 independent simulations and record the fraction in which the algorithm recovers the correct ranking. \cref{tab:g6-fixed-cap} reports the results. For each $n$, we report the fraction of simulations in which the algorithm obtains the correct ranking. The "Min. score gap" column indicates the minimum difference between the algorithm's expected scores between any two adjacent models in the ranking. A negative value indicates misranking. $\cA_{\mathrm{MLE-AVG}}$ has a negative gap in both populations, so its population expected scores misorder at least one pair. Its full-ranking recovery probability therefore tends to zero as the number of contexts grows. 
$\cA_{\mathrm{MLE-UNCONST}}$ and $\cA_{\mathrm{MLE-TRR}}$ have identical population expected scores by design at a common repetition budget but different variances in the estimates.
Their differing recovery frequencies at the same $n$ illustrate that equal repetition budgets need not imply the same performance on finite contexts.
\end{document}